\documentclass{article}
\usepackage{iclr2027/iclr2027_conference,times}

\usepackage[T1]{fontenc}
\usepackage{microtype}
\usepackage{graphicx}
\usepackage{subcaption}
\usepackage{booktabs} % for professional tables
\usepackage{array}
\usepackage{longtable}
\usepackage{needspace}

\usepackage{hyperref}
\usepackage{url}

\usepackage{amsmath}
\usepackage{amssymb}
\usepackage{mathtools}
\usepackage{amsthm}

\usepackage[capitalize,noabbrev]{cleveref}

\theoremstyle{plain}
\newtheorem{theorem}{Theorem}[section]
\newtheorem{proposition}[theorem]{Proposition}

\newtheorem{corollary}[theorem]{Corollary}
\theoremstyle{definition}

\theoremstyle{remark}

\usepackage[textsize=tiny]{todonotes}

\usepackage{algorithm}
\usepackage{algorithmic}
\usepackage{xcolor}
\usepackage{colortbl}
\usepackage{multirow}
\usepackage{enumitem}
\usepackage{listings}
\usepackage{float}
\usepackage{wrapfig}
\usepackage{fvextra}
\DefineVerbatimEnvironment{MintedVerbatim}{Verbatim}{}
\input{minted-cache/default.style.minted}
\renewcommand{\PYGZsq}{\textquotesingle}
\renewcommand{\PYGZhy}{\mbox{-}}
\VerbatimPygments{\PYG}{\PYG}
\usepackage{tcolorbox}
\tcbuselibrary{breakable,skins}

\definecolor{bestgray}{gray}{0.9}

\title{Fusion is the New Mutation: Bandit-Guided Evolution on Workflow Graphs}

\author{%
\parbox{\dimexpr\textwidth-2\tabcolsep\relax}{\raggedright\fontsize{9}{10}\selectfont
\textbf{Zhiwei Shang$^{1}$, Jiahang Sun$^{1}$, Mingrong Gong$^{2}$, Mingze Kong$^{1}$,}\\
\textbf{Zikun Qu$^{1}$, Pingchen Lu$^{1}$, Junhao Dong$^{2}$, Zhipiao Liu$^{3}$,}\\
\textbf{Hongwei Yang$^{3}$, Guoqing Xie$^{3}$, Yao Shu$^{4}$, Zhongxiang Dai$^{1}$}\\
\normalfont\fontsize{8}{9}\selectfont
$^{1}$School of Data Science, The Chinese University of Hong Kong, Shenzhen, China\\
$^{2}$College of Computing and Data Science, Nanyang Technological University, Singapore\\
$^{3}$Huawei Technologies Co., Ltd., China\\
$^{4}$Artificial Intelligence Thrust, The Hong Kong University of Science and Technology (Guangzhou), China\\
Correspondence to: Zhongxiang Dai \texttt{<daizhongxiang@cuhk.edu.cn>}
}}

\hypersetup{pdftitle={Fusion is the New Mutation: Bandit-Guided Evolution on Workflow Graphs},
 pdfauthor={Zhiwei Shang, Jiahang Sun, Mingrong Gong, Mingze Kong, Zikun Qu, Pingchen Lu, Junhao Dong, Zhipiao Liu, Hongwei Yang, Guoqing Xie, Yao Shu, Zhongxiang Dai},
 pdfkeywords={Workflow Optimization, Multi-Armed Bandits, Large Language Models, Evolutionary Algorithms}}
\iclrfinalcopy
\makeatletter
\renewcommand{\@maketitle}{%
  \vspace*{-42pt}%
  \vbox{\hsize\textwidth
    {\LARGE\scshape\@title\par}%
    \vskip 7pt
    \noindent\hspace*{\tabcolsep}\@author\par
    \vskip 3pt}}
\makeatother

\begin{document}
\maketitle
\suppressfloats[t]

\begin{abstract}
Automated agentic workflow optimization relies on costly evaluations, making it essential to allocate a limited evaluation budget effectively. Multi-parent fusion can reuse designs from previously discovered workflows, but identifying promising parent combinations requires learning from limited fusion feedback. We introduce \textbf{DAGO} (Directed Acyclic Graph Optimization), a contextual-bandit-guided framework that learns which parent workflows to fuse under a limited evaluation budget. DAGO formulates each candidate parent combination as an \textbf{arm}, represented by pretrained embeddings of its constituent workflows' code and prompts. A diagonal LinUCB policy learns a shared reward model across arms and balances exploitation of arms with high predicted offspring quality against uncertainty-driven exploration. After an arm is selected, an LLM generates a child workflow through summary-guided fusion, and the child's validation score serves as the reward for updating the bandit. A shared directed acyclic graph maintains discovered workflows and their multi-parent lineage, providing an expanding pool of parents for subsequent arm proposals. Across six benchmarks covering mathematical reasoning, code generation, and question answering, DAGO achieves the highest macro-average score among the evaluated baselines. Under matched validation-evaluation budgets, it improves over AFlow from 80.3 to 81.7 while reducing aggregate search expenditure by 11.2\%. Ablation studies show that LinUCB-guided arm selection outperforms both random selection and its exploration-free variant, supporting the value of feedback-driven selection and exploration--exploitation balance.
\end{abstract}

\begin{figure*}[t]
\centering
\includegraphics[width=0.96\textwidth]{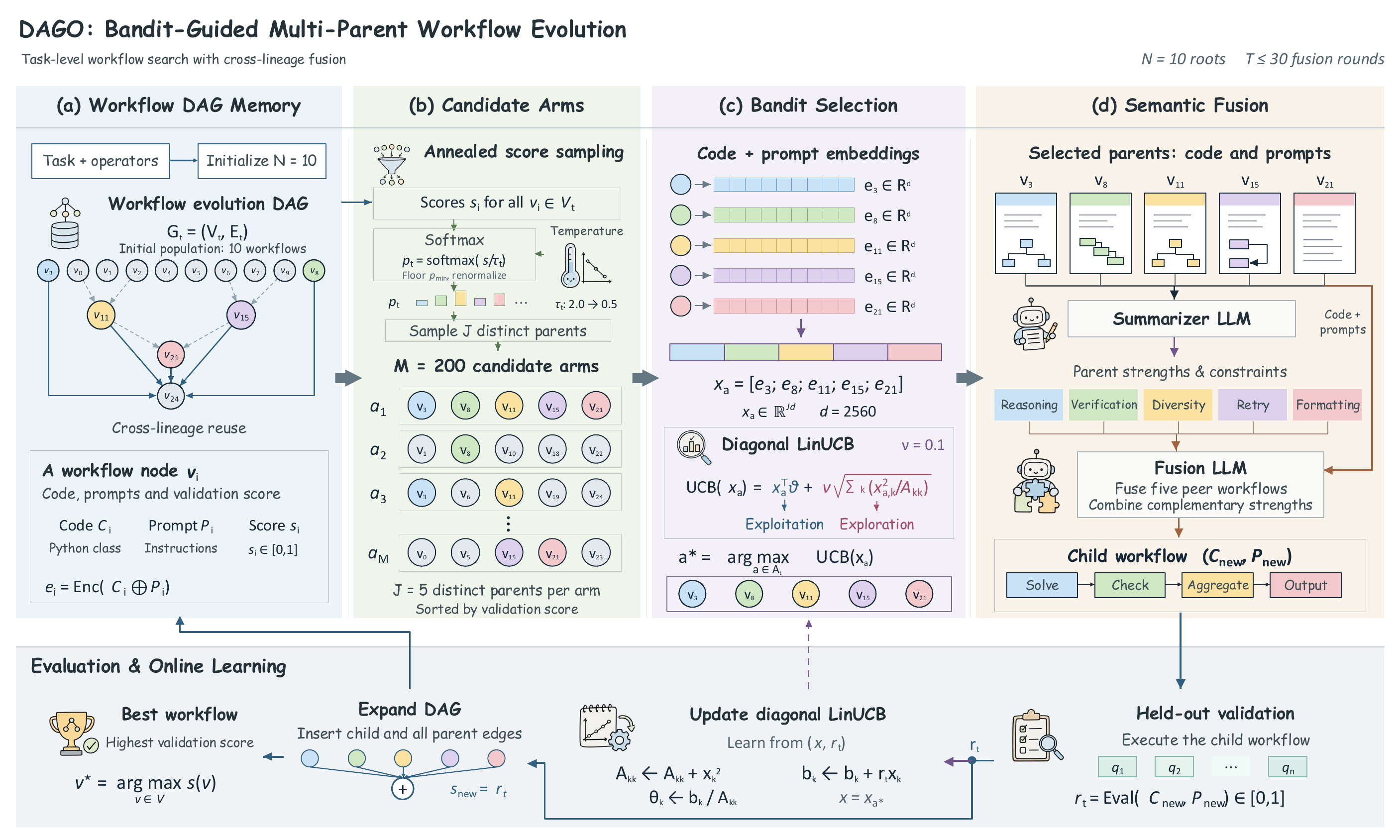}
\captionsetup{font=footnotesize}
\caption{\textbf{Overview of the DAGO framework.}
\textbf{(a)} The DAG stores workflows, validation scores, and parent lineage.
\textbf{(b)} Annealed score sampling proposes parent combinations.
\textbf{(c)} Diagonal LinUCB selects a combination from code-and-prompt embeddings.
\textbf{(d)} Summary-guided fusion generates a child.
Validation feedback updates LinUCB and the DAG; the highest-scoring workflow is returned.}
\label{fig:system}
\end{figure*}

\section{Introduction}
\label{sec:introduction}

Agentic workflows coordinate large language model (LLM) calls, tools, and control-flow logic to solve complex tasks.
Automating their design turns workflow construction into a search problem over executable programs.
ADAS generates new agentic systems by conditioning on an archive of previous discoveries~\citep{hu2025automated}, while AFlow explores workflow code through a variant of Monte Carlo Tree Search~\citep{zhang2025aflow}.
In these settings, assessing a candidate requires executing its LLM calls and tools on validation examples, making search evaluations costly.
Under a limited evaluation budget, an optimizer must therefore decide not only how to generate new workflows, but also which generation decisions are worth evaluating.

Multi-parent fusion offers a way to reuse useful designs from previously discovered workflows.
For example, a workflow with effective problem decomposition may provide useful guidance for improving another workflow with a strong verification procedure.
Existing approaches already explore workflow recombination: EvoFlow uses LLM-based crossover~\citep{zhang2025evoflow}, and MermaidFlow introduces constraint-preserving crossover over workflow graphs~\citep{zheng2025mermaidflow}.
However, an effective fusion operator does not by itself determine which parents should be fused.
High individual validation scores do not guarantee that a parent combination will produce a strong child, and the number of possible combinations grows combinatorially as more workflows are discovered.
Evaluating every combination is impractical, while selecting parent combinations at random does not use past fusion outcomes to guide future choices.
The central challenge is thus to learn which combinations merit the next expensive fusion and evaluation.

This motivates our research question: \emph{Under a limited evaluation budget, how can we learn from past fusion outcomes to select promising parent workflow combinations?}
We introduce \textbf{DAGO} (Directed Acyclic Graph Optimization), a contextual-bandit-guided framework that makes this selection problem explicit.
An \emph{arm} is a selectable action in a bandit problem; here, it corresponds to a candidate combination of parent workflows to fuse.
The child's validation score serves as the reward, rather than an aggregate of the parents' existing scores.
This distinction directs learning toward the quality of the generated workflow, rather than merely favoring parents that already perform well individually.

To learn from limited feedback, DAGO uses a \emph{shared linear reward model} across arms instead of learning an independent reward estimate for every combination.
Each workflow is represented by pretrained embeddings of its code and prompts, and an arm's features concatenate its parent embeddings in descending order of validation score.
At each iteration, annealed score-based sampling proposes a manageable set of candidate arms without enumerating the full combination space.
A diagonal LinUCB policy~\citep{li2010contextual,chu2011contextual} then selects an arm by combining predicted child-workflow quality with an uncertainty-driven exploration bonus.
The shared representation allows feedback from evaluated arms to inform the ranking of newly proposed arms, while the exploration bonus encourages evaluating alternatives beyond those with the highest current predictions.
The diagonal approximation keeps arm scoring and online updates lightweight.
Thus, parent scores guide \emph{candidate proposal}, whereas observed fusion rewards guide \emph{arm selection}, allocating evaluation opportunities through an explicit exploration--exploitation mechanism.

DAGO embeds this decision process in a workflow-generation loop.
A shared directed acyclic graph (DAG) stores discovered workflows and their multi-parent generation lineage, providing an expanding pool of parents for subsequent arm proposals.
This DAG records the search history, not the execution graph within an individual workflow.
After the bandit selects an arm, an LLM summarizes the parents' strengths and constraints and generates a child through summary-guided fusion, starting from the strongest parent and applying targeted modifications informed by the selected workflows.
The child's validation score updates the bandit, and the child is added to the DAG.
At the end of search, DAGO returns the workflow with the highest validation score for reuse on the target task; arm selection occurs during optimization rather than separately for each test query.

\begin{figure}[!t]
\centering
\captionsetup{font=footnotesize}
\begin{subfigure}[b]{0.63\textwidth}
\centering
\includegraphics[width=\linewidth]{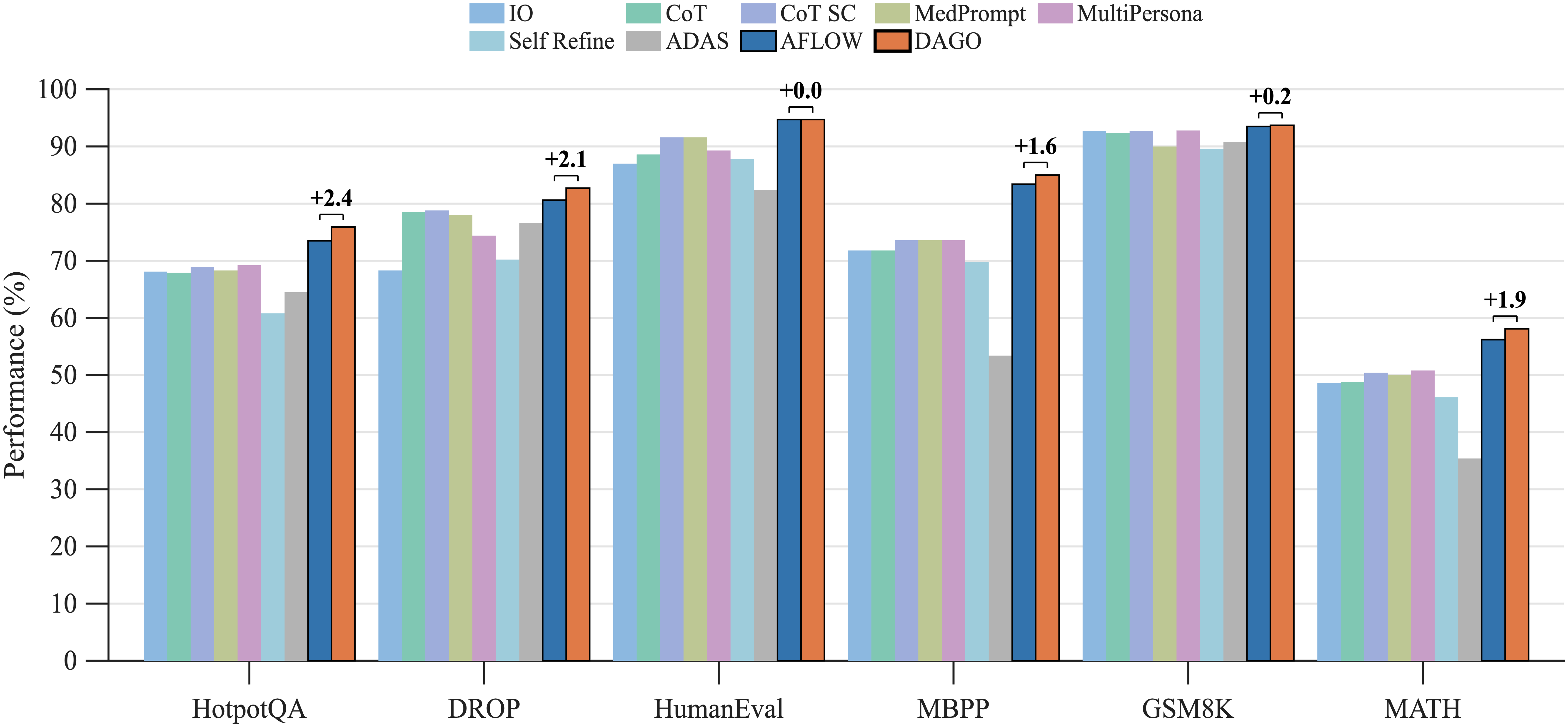}
\caption{Performance across benchmarks.}
\end{subfigure}\hfill
\begin{subfigure}[b]{0.35\textwidth}
\centering
\includegraphics[width=\linewidth]{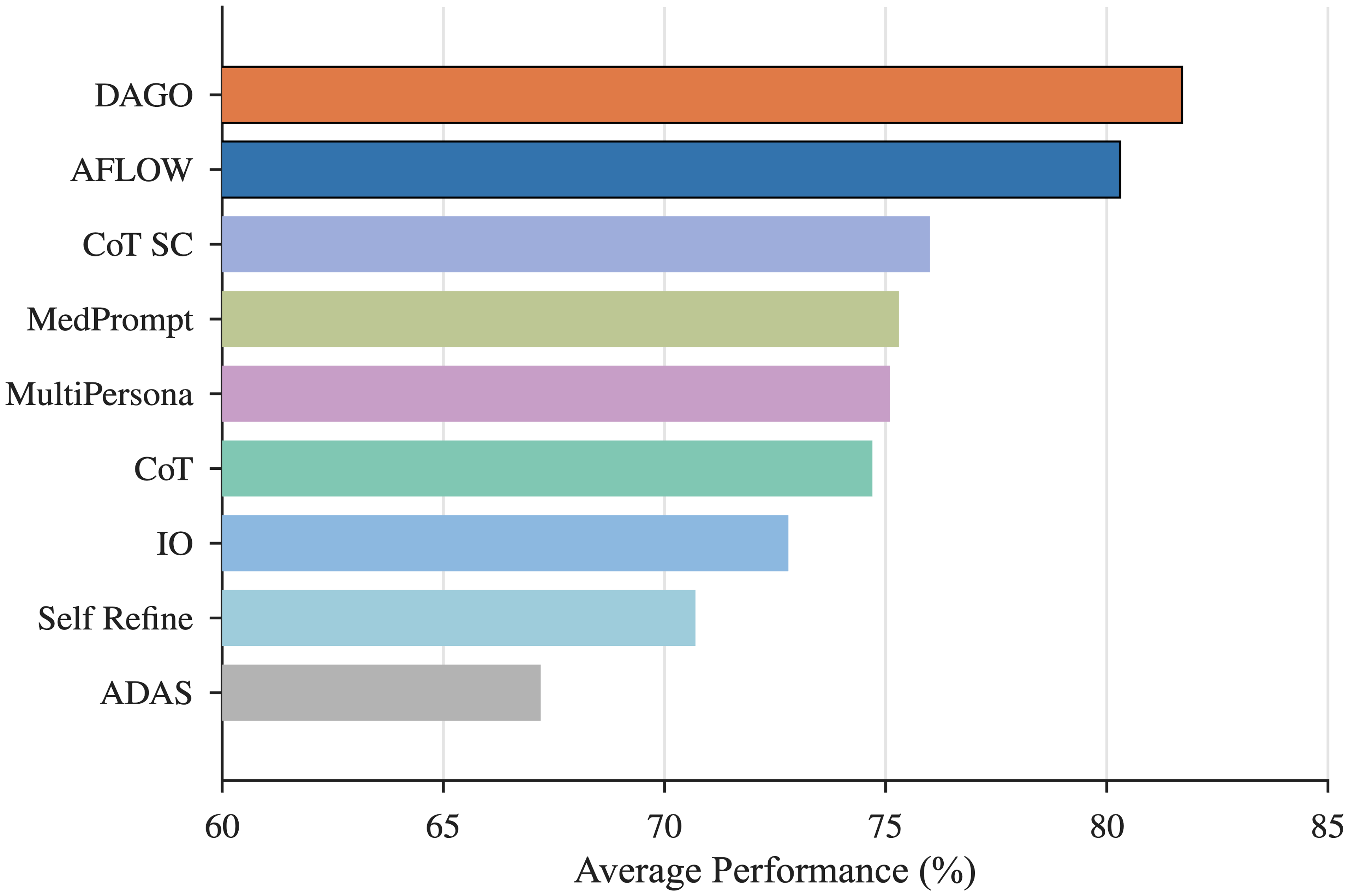}
\caption{Macro-average ranking.}
\end{subfigure}
\caption{\textbf{Overall performance across benchmarks and average ranking.}
Selected methods from Table~\ref{tab:main}, with test scores averaged over three executions.
(a) Solve rate / pass@1 / F1 for math / code / QA; brackets show DAGO's gains over AFlow in points.
(b) Six-task macro averages.}
\label{fig:main_results}
\end{figure}

We evaluate DAGO on six benchmarks spanning mathematical reasoning (GSM8K and MATH), code generation (HumanEval and MBPP), and question answering (HotpotQA and DROP).
DAGO achieves the highest macro-average score among the evaluated baselines.
Under the same validation-evaluation budget cap as AFlow, it improves the macro-average score from $80.3$ to $81.7$ while reducing aggregate search expenditure by $11.2\%$ across the six benchmarks.
Arm-selection ablations show that the default LinUCB policy achieves higher mean scores than both random selection and an exploration-free variant on all three tested benchmarks, supporting the value of feedback-guided selection and the exploration bonus.
Our contributions are threefold:
\begin{enumerate}[nosep]
    \item \textbf{A bandit formulation of workflow fusion decisions.}
    We model each parent workflow combination as an arm whose reward is the generated child's validation score, making feedback-driven parent selection central to workflow search under limited evaluation budgets.

    \item \textbf{A shared, uncertainty-aware arm-selection mechanism.}
    We combine pretrained code-and-prompt representations, score-based candidate proposal, and diagonal LinUCB to reuse fusion feedback across arms and guide exploration without exhaustive combination evaluation.

    \item \textbf{An integrated framework with empirical evaluation.}
    We integrate arm selection with DAG-based workflow memory and summary-guided fusion, and evaluate performance, search expenditure, and the effects of learned selection and exploration across reasoning, coding, and question-answering tasks.
\end{enumerate}

\section{Related Work}
\label{sec:related_work}

\paragraph{Automated Workflow Optimization and Recombination.}
Automated workflow optimization spans prompt refinement and
structural search.
DSPy~\citep{khattab2024dspy},
TextGrad~\citep{yuksekgonul2025optimizing}, and
MIPRO~\citep{opsahl-ong-etal-2024-optimizing} improve prompts in
modular LLM programs, while GPTSwarm~\citep{zhuge2024gptswarm} and
G-Designer~\citep{zhang2025gdesigner} optimize multi-agent graphs.
ADAS~\citep{hu2025automated} searches executable programs using
an archive of discoveries, and AFlow~\citep{zhang2025aflow}
refines code workflows through an MCTS variant with score-based
mixed-probability selection.
JudgeFlow~\citep{ma2026judgeflow} targets problematic blocks using
execution diagnostics.
ScoreFlow~\citep{wang2025scoreflow} and
FlowReasoner~\citep{gao2025flowreasoner} train workflow generators
through preference optimization and reinforcement learning,
respectively.
Multi-parent recombination is also explored:
EvoFlow~\citep{zhang2025evoflow} retrieves query-relevant parents
for LLM-based crossover, whereas
MermaidFlow~\citep{zheng2025mermaidflow} applies
constraint-preserving crossover to workflow graphs.
Building on these approaches, DAGO focuses on
\emph{learning which parent combinations to fuse} from past
fusion outcomes, allocating a limited evaluation budget to
promising generation decisions rather than relying solely on
parent scores or query relevance.

\paragraph{Bandit-Guided Selection under Limited Evaluation Budgets.}
Bandits provide a framework for allocating evaluations through
exploration--exploitation.
OPTS~\citep{ashizawa-etal-2025-bandit} applies Thompson sampling
to prompt-design strategy selection, while
TRIPLE~\citep{shi2024efficient} formalizes prompt selection as
fixed-budget best-arm identification.
MBA-RAG~\citep{tang-etal-2025-mba} selects retrieval strategies,
and PILOT~\citep{panda-etal-2025-adaptive} uses contextual bandits
for budget-constrained model routing.
Linear contextual bandits~\citep{li2010contextual,chu2011contextual}
leverage observed features for reward prediction and
uncertainty-aware selection.
DAGO applies this principle to workflow fusion: each
\emph{arm} corresponds to a parent workflow combination, and
its reward is the generated child's validation score, not the
parents' existing scores.
Unlike query-time routing among existing models, arm selection
guides the generation of new workflows during task-level search.
As newly discovered workflows introduce additional arms, the
selector must learn from sparse feedback in an expanding
search space.
A shared linear reward model over pretrained code-and-prompt
embeddings allows feedback from evaluated arms to inform newly
proposed arms, while diagonal LinUCB balances predicted child
quality with uncertainty-driven exploration.
This addresses the growing combination space without separately
learning each arm's reward or exhaustively evaluating all
combinations.

\section{Preliminaries}
\label{sec:preliminary}

\paragraph{Workflows and validation feedback.}
A workflow $w=(\mathcal{C},\mathcal{P})$ comprises executable orchestration code and its internal prompts.
Following the code-based representation of AFlow~\citep{zhang2025aflow}, $\mathcal{C}$ implements an asynchronous \texttt{\_\_call\_\_} method that maps a task instance to an output through LLM calls, tools, and control flow.
Let $\mathrm{Eval}_{q}(w)\in[0,1]$ be its validation score averaged over $q$ executions of the fixed validation set.
A stored score $s(v)$ is an observed evaluation, not the unknown expected performance of that workflow.
Validation feedback is available during search; test data are not used for generation, arm selection, or stopping.

\paragraph{Workflow archive and fusion arms.}
The archive is a DAG $\mathcal{G}_t=(\mathcal{V}_t,\mathcal{E}_t)$ whose nodes store code, prompts, scores, and cached embeddings.
An edge records that a parent was supplied when generating a child; it does not certify that a specific component was inherited.
The DAG records generation lineage, not workflow execution.
Each \emph{arm} $a=(v_{(1)},\ldots,v_{(J)})$ specifies $J$ distinct parents, canonically ordered by decreasing score, with node identifiers breaking ties.
Let $\bar{\mathcal{A}}_t$ contain all such arms available before round $t$, and let $\mathcal{A}_t\subseteq\bar{\mathcal{A}}_t$ be the proposed candidates.
We use a fixed arm size with $N\ge J$ initial workflows.
Canonical ordering removes arbitrary permutations of the same parent set, so the full arm space contains $\binom{|\mathcal{V}_{t-1}|}{J}$ choices before round $t$.

\paragraph{Budgeted, partially observed decisions.}
Pulling arm $a_t$ generates one child $w_t$ through LLM summarization and fusion, followed by validation with reward $r_t=\mathrm{Eval}_{q_t}(w_t)$.
Conditional on the information $\mathcal{H}_t$ available before selection, define
\begin{equation}
\mu_t(a)=\mathbb{E}[r_t(a)\mid\mathcal{H}_t],
\qquad
Nq_{\mathrm{init}}+\sum_{t=1}^{K}q_t\le B_{\max},
\label{eq:budget_feedback}
\end{equation}
where $r_t(a)$ is the potential reward of generating and evaluating a child from $a$, and $K$ is the number of completed fusion rounds.
Only the selected arm's reward is observed.
We do not assume that $\mu_t$ is linear in arm features or stationary across rounds.
The objective is to discover a high-scoring workflow within the evaluation budget and return $\arg\max_{v\in\mathcal{V}_K}s(v)$; per-round reward prediction is a search surrogate, not an equivalent best-workflow guarantee.

\section{Methodology}
\label{sec:method}

\textbf{DAGO} (Directed Acyclic Graph Optimization) learns which parent workflow combinations to evaluate through a shared, LinUCB-style selection rule.
Each round separates inexpensive candidate proposal and scoring from expensive child generation and evaluation.
The DAG supplies reusable parents, the fusion operator generates a child, and bandit feedback guides which arm receives the next evaluation.
Lineage is retained for provenance; graph topology itself is not an additional input to the arm scorer in this implementation.

\paragraph{Initialization and candidate proposal.}
Task-specific strategy discovery initializes $N$ workflows, each scored with $q_{\mathrm{init}}$ validation executions and embedded once before insertion into $\mathcal{G}_0$.
Initial scores guide parent proposal but are not treated as rewards of unexecuted fusion arms.
Before round $t$, let $s_i=s(v_i)$ for $v_i\in\mathcal{V}_{t-1}$.
DAGO computes
\begin{equation}
p_t(v_i)=\frac{\exp(s_i/\tau_t)}{\sum_j\exp(s_j/\tau_t)},
\qquad
\widetilde p_t(v_i)=
\frac{\max\{p_t(v_i),p_{\min}\}}{\sum_j\max\{p_t(v_j),p_{\min}\}},
\label{eq:proposal_distribution}
\end{equation}
with $\tau_t=\tau_{\mathrm{init}}+(\tau_{\mathrm{final}}-\tau_{\mathrm{init}})t/T$.
At round $t$, set $M_t=\min\{M,\binom{|\mathcal{V}_{t-1}|}{J}\}$. DAGO repeatedly draws $J$ parents \emph{without replacement within each arm}, canonically orders them, and rejects duplicates until $M_t$ unique arms are collected.
The floor preserves positive parent probabilities after renormalization, but does not guarantee that a useful arm appears within a finite budget.
Candidate generation uses existing scores; it requires no new workflow evaluation.

\paragraph{Shared arm representation.}
A fixed encoder maps each node to $\mathbf{e}_i=\mathrm{Enc}(\mathcal{C}_i\oplus\mathcal{P}_i)\in\mathbb{R}^d$, where $\oplus$ concatenates text.
DAGO caches these embeddings and represents an arm as $\mathbf{x}(a)=[\mathbf{e}_{(1)};\ldots;\mathbf{e}_{(J)}]\in\mathbb{R}^{p}$, with $p=Jd$.
Consistent ordering distinguishes the strongest parent from auxiliary parents.
A shared linear score transfers feedback across arms instead of maintaining an independent estimate for every combination.
Each observed child updates weights subsequently used to score other arms, including newly formed arms whose parents share feature coordinates with earlier candidates.
Only one of the $M_t$ proposed arms yields a fusion-reward observation per round, motivating reuse of a fixed representation rather than training a new encoder from these sparse observations.
This is a lightweight \emph{reward surrogate}: its additive form does not explicitly model interactions between different parents, and pretrained embeddings do not make LLM fusion linearly realizable.

\paragraph{Diagonal LinUCB-style arm selection.}
Inspired by linear contextual bandits~\citep{li2010contextual,chu2011contextual}, DAGO maintains diagonal statistics $\mathbf{D}_{t-1}=\mathrm{diag}(d_{t-1,1},\ldots,d_{t-1,p})$ and $\mathbf{b}_{t-1}$, initialized as $\mathbf{D}_0=\lambda\mathbf{I}$ and $\mathbf{b}_0=\mathbf{0}$.
It computes $\widehat{\boldsymbol{\theta}}_{t-1}=\mathbf{D}_{t-1}^{-1}\mathbf{b}_{t-1}$ and selects
\begin{equation}
\begin{split}
a_t&\in\arg\max_{a\in\mathcal{A}_t} U_t(a),\\
U_t(a)&=\underbrace{\mathbf{x}(a)^\top\widehat{\boldsymbol{\theta}}_{t-1}}_{\widehat\mu_t(a)}
+\nu\underbrace{\sqrt{\sum_{k=1}^{p}x_k(a)^2/d_{t-1,k}}}_{h_t(a)}.
\end{split}
\label{eq:ucb}
\end{equation}
The first term ranks arms using past fusion rewards; the second favors features with less accumulated coordinate-wise coverage.
Here $h_t$ is an exploration proxy, not a calibrated confidence interval for nonlinear fusion rewards.
Moreover, replacing the full Gram matrix by its diagonal changes both the fitted coefficients and the exploration term; the update is not ordinary multivariate ridge regression.
An arm with individually weaker parents can outrank another candidate when its learned score or exploration bonus is larger.
However, the selector cannot recover a promising arm omitted by proposal.
Scoring $M_t$ arms costs $O(M_t p)$ and storing the bandit state costs $O(p)$.

\paragraph{Fusion, evaluation, and feedback update.}
For each selected parent, a summarization LLM produces $z_j=\mathrm{Summarize}(\mathcal{C}_j,\mathcal{P}_j,s_j)$ describing strengths, weaknesses, and interface constraints.
The fusion LLM receives \emph{both} the parents and their summaries:
\begin{equation}
w_t=\mathrm{Fuse}\!\left(\{(\mathcal{C}_j,\mathcal{P}_j,s_j,z_j)\}_{v_j\in a_t}\right),
\qquad r_t=\mathrm{Eval}_{q_t}(w_t).
\label{eq:fusion_reward}
\end{equation}
The prompt starts from the strongest parent and permits a small number of justified, targeted modifications informed by the selected workflows.
The reward is the child's absolute validation score, \emph{not} improvement over a parent.
With $\mathbf{x}_t=\mathbf{x}(a_t)$, DAGO updates
\begin{equation}
 d_{t,k}=d_{t-1,k}+x_{t,k}^2,
 \qquad b_{t,k}=b_{t-1,k}+r_t x_{t,k},
 \quad k=1,\ldots,p.
\label{eq:diagonal_update}
\end{equation}
The evaluated child is embedded and inserted with incoming edges from its selected parents.
These sufficient statistics implement the feedback update; no periodic refitting or forgetting step is required by this procedure.

\paragraph{Budget and stopping.}
Algorithm~\ref{alg:dago} returns the highest-scoring stored workflow.
The reported configuration evaluates $N=10$ roots once and each child three times, allowing at most $T=30$ fusion rounds within $B_{\max}=100$ validation-set evaluations.
After at least $T_{\min}=20$ fusion rounds, early stopping checks for a lack of relative improvement in the most recent $C_{\mathrm{patience}}=5$ round scores; this recent-score criterion is separate from retaining the highest-scoring workflow.
Repeated executions are averaged into \emph{one} reward per selected arm, not three independent fusion decisions.
The budget therefore measures validation executions, not proposed arms or generated children alone; LLM generation and summarization costs are accounted for separately.
Appendix~\ref{app:algorithm} specifies sampling, stopping, and implementation details.

\begin{algorithm}[t]
\caption{DAGO: feedback-guided fusion under an evaluation budget}
\label{alg:dago}
\footnotesize
\begin{algorithmic}[1]
\REQUIRE $N\ge J$, $M\ge1$; $T,B_{\max}$; repeats $q_{\mathrm{init}},q$; $\lambda,\nu$; stopping parameters
\STATE Generate, evaluate, and embed $N$ roots; initialize $\mathcal{G}_0$
\STATE $\mathbf{D}\leftarrow\lambda\mathbf{I}_{Jd}$, $\mathbf{b}\leftarrow\mathbf{0}$; $B\leftarrow Nq_{\mathrm{init}}$
\STATE Set incumbent $v^*\leftarrow\arg\max_{v\in\mathcal{V}_0}s(v)$; initialize empty reward history
\FOR{$t=1$ to $T$}
    \STATE \textbf{if} $B+q>B_{\max}$ \textbf{then break}
    \STATE Set $M_t\leftarrow\min\{M,\binom{|\mathcal{V}_{t-1}|}{J}\}$; sample $M_t$ unique arms using Eq.~\eqref{eq:proposal_distribution}
    \STATE Select $a_t$ using Eq.~\eqref{eq:ucb}; set $\mathbf{x}_t\leftarrow\mathbf{x}(a_t)$
    \STATE Summarize selected parents; fuse them into child $w_t$
    \STATE $r_t\leftarrow\mathrm{Eval}_{q}(w_t)$; $B\leftarrow B+q$
    \STATE Update $(\mathbf{D},\mathbf{b})$ by Eq.~\eqref{eq:diagonal_update}
    \STATE Embed and insert child with parent edges; update $v^*$ and record $r_t$
    \STATE \textbf{if} $t\ge T_{\min}$ and recent scores satisfy Eq.~\eqref{eq:early_stopping_gate} \textbf{then break}
\ENDFOR
\STATE \textbf{return} $v^*$
\end{algorithmic}
\end{algorithm}

\paragraph{Decision analysis without linear realizability.}
We analyze the selector's actual scores rather than transferring a full-matrix linear-bandit guarantee to the LLM pipeline.
Let $a_t^\star\in\arg\max_{a\in\mathcal{A}_t}\mu_t(a)$ and define the residual span
$\varepsilon_t=\max_{a\in\mathcal{A}_t}(\mu_t(a)-\widehat\mu_t(a))-\min_{a\in\mathcal{A}_t}(\mu_t(a)-\widehat\mu_t(a))$.
For arbitrary, possibly nonlinear $\mu_t$, the maximizing rule implies
\begin{equation}
\mu_t(a_t^\star)-\mu_t(a_t)
\le \varepsilon_t+\nu\bigl(h_t(a_t)-h_t(a_t^\star)\bigr)
\le \varepsilon_t+\nu h_t(a_t).
\label{eq:decision_error}
\end{equation}
Appendix~\ref{app:theory} proves this inequality, separates candidate-proposal loss, and bounds accumulated diagonal exploration bonuses.
The residual span is unchanged by a common offset in all predictions, emphasizing relative discrimination rather than perfect reward calibration.
The result identifies prediction differences and proposal coverage as conditions for effective selection; it does not assume they improve or establish sublinear regret for DAGO.
The practical value of feedback and exploration is assessed by the random-selection and exploration-free ablations, rather than inferred from linear realizability.

\section{Experiments}
\label{sec:experiments}

\subsection{Experimental Setup}
\paragraph{Benchmarks.}
We evaluate DAGO on six benchmarks: GSM8K~\citep{cobbe2021training} and MATH~\citep{hendrycksmath2021} for mathematical reasoning, HumanEval~\citep{chen2021codex} and MBPP~\citep{austin2021program} for code generation, and HotpotQA~\citep{yang-etal-2018-hotpotqa} and DROP~\citep{dua-etal-2019-drop} for question answering.
We follow AFlow's dataset splits and evaluation protocols~\citep{zhang2025aflow}, restricting MATH to level-5 problems.
Validation data are used for search and workflow selection; final results are reported on the disjoint test splits.

\paragraph{Baselines.}
We compare with prompting and answer-refinement methods: IO, CoT~\citep{wei2022chain}, CoT-SC~\citep{wang2023selfconsistency}, MedPrompt~\citep{nori2023can}, MultiPersona~\citep{wang-etal-2024-unleashing}, Self-Refine~\citep{madaan2023selfrefine}, and GoT~\citep{besta2024graph}.
We also include answer ensembling with Blender~\citep{jiang2023llmblender}, dynamic agent collaboration with DyLAN~\citep{liu2024dylan}, and automated workflow optimization with GPTSwarm~\citep{zhuge2024gptswarm}, ADAS~\citep{hu2025automated}, MaAS~\citep{zhang2025multiagent}, and AFlow~\citep{zhang2025aflow}.
AFlow is our main reference for single-parent workflow search.
Blender aggregates candidate answers at inference time, whereas DAGO fuses executable workflows during search.
% AUTHOR CHECK: Confirm that "Blender" is the LLM-Blender implementation cited here.
% Document the source, executor(s), and any adaptations of each newly added baseline.

\paragraph{Metrics.}
We report solve rate (\%) on GSM8K and MATH, pass@1 (\%) on HumanEval and MBPP~\citep{chen2021codex}, and F1 (\%) on HotpotQA and DROP.
The overall score is the unweighted mean across benchmarks.
Unless otherwise stated, each validation-selected workflow is fixed and executed three times on the same held-out test split; we report the mean test score. These repetitions are test executions, not independent optimization runs.

\paragraph{Implementation Details.}
DAGO follows AFlow's fixed model configuration: Claude~3.5~Sonnet~\citep{claude35} for workflow generation, summarization, and fusion, and GPT-4o-mini-0718~\citep{gpt4o} for execution.
Both methods use identical validation/test splits and a budget of up to 100 validation-set evaluations per run ($20\times5$ for AFlow; $10+30\times3$ for DAGO).
Workflow selection uses validation scores only; repeated test evaluations are performed after selection and are not used to revise the workflow.
Qwen3-Embedding~\citep{zhang2025qwen3} produces $2560$-dimensional workflow embeddings.
We set $N{=}10$, $T{=}30$, $M{=}200$, $J{=}5$, and $\nu{=}0.1$, annealing the sampling temperature from $2.0$ to $0.5$.
Full hyperparameters, early-stopping rules, and dataset details appear in Appendices~\ref{app:hyperparameters} and~\ref{app:datasets}.

\subsection{Overall Performance and Analysis}

\begin{table*}[t]
\captionsetup{font=footnotesize}
\caption{\textbf{Main results on six benchmarks.}
Means over three test executions (\%), grouped by task category. Avg. is the six-task macro average; best scores are \textbf{bold}.}
\label{tab:main}
\centering
\small
\setlength{\tabcolsep}{5.2pt}
\renewcommand{\arraystretch}{0.90}
\begin{tabular}{@{}lcc@{\hspace{6pt}}cc@{\hspace{6pt}}cc|c@{}}
\toprule
& \multicolumn{2}{c}{\textbf{QA}} & \multicolumn{2}{c}{\textbf{Code}} & \multicolumn{2}{c}{\textbf{Math}} & \\
\cmidrule(lr){2-3}\cmidrule(lr){4-5}\cmidrule(lr){6-7}
\textbf{Method} & HotpotQA & DROP & HumanEval & MBPP & GSM8K & MATH & Avg. \\
\midrule
IO           & 68.1 & 68.3 & 87.0 & 71.8 & 92.7 & 48.6 & 72.8 \\
CoT~\cite{wei2022chain}          & 67.9 & 78.5 & 88.6 & 71.8 & 92.4 & 48.8 & 74.7 \\
CoT-SC~\cite{wang2023selfconsistency}       & 68.9 & 78.8 & 91.6 & 73.6 & 92.7 & 50.4 & 76.0 \\
MedPrompt~\cite{nori2023can}    & 68.3 & 78.0 & 91.6 & 73.6 & 90.0 & 50.0 & 75.3 \\
MultiPersona~\cite{wang-etal-2024-unleashing} & 69.2 & 74.4 & 89.3 & 73.6 & 92.8 & 50.8 & 75.1 \\
Self-Refine~\cite{madaan2023selfrefine}  & 60.8 & 70.2 & 87.8 & 69.8 & 89.6 & 46.1 & 70.7 \\
GoT~\cite{besta2024graph} & 72.5 & 71.7 & 86.3 & 68.6 & 86.7 & 39.5 & 70.9 \\
Blender~\cite{jiang2023llmblender} & 72.3 & 80.4 & 88.7 & 77.1 & 88.4 & 46.9 & 75.6 \\
DyLAN~\cite{liu2024dylan} & 74.1 & 82.2 & 90.4 & 77.3 & 90.0 & 48.5 & 77.1 \\
GPTSwarm~\cite{zhuge2024gptswarm} & 73.2 & 81.0 & 89.3 & 77.4 & 89.1 & 47.9 & 76.3 \\
ADAS~\cite{hu2025automated}         & 64.5 & 76.6 & 82.4 & 53.4 & 90.8 & 35.4 & 67.2 \\
MaAS~\cite{zhang2025multiagent} & 75.3 & \textbf{83.1} & 91.6 & 82.2 & 91.5 & 52.2 & 79.3 \\
AFlow~\cite{zhang2025aflow}        & 73.5 & 80.6 & \textbf{94.7} & 83.4 & 93.5 & 56.2 & 80.3 \\
\midrule
DAGO         & \textbf{75.9} & 82.7 & \textbf{94.7} & \textbf{85.0} & \textbf{93.7} & \textbf{58.1} & \textbf{81.7} \\
\bottomrule
\end{tabular}
\end{table*}

\definecolor{bestgray}{gray}{0.92} % light gray for highlighting best scores

\paragraph{Overall Performance.}
As shown in \cref{tab:main,fig:main_results}, DAGO achieves the highest macro-average score of $81.7$, compared with $80.3$ for AFlow and $79.3$ for MaAS.
It ranks first on HotpotQA, MBPP, GSM8K, and MATH, and ties AFlow on HumanEval.
The gains over AFlow are largest on HotpotQA ($+2.4$ points) and DROP ($+2.1$), followed by MATH ($+1.9$) and MBPP ($+1.6$); the gain on GSM8K is smaller ($+0.2$).
MaAS remains strongest on DROP, scoring $83.1$ versus DAGO's $82.7$.
Compared with MaAS, DAGO's largest gains are on MATH ($+5.9$), HumanEval ($+3.1$), and MBPP ($+2.8$).
Thus, DAGO's average advantage comes from improvements across task categories, rather than a gain on a single benchmark.

\begin{wrapfigure}{l}{0.38\textwidth}
\centering
\includegraphics[width=\linewidth]{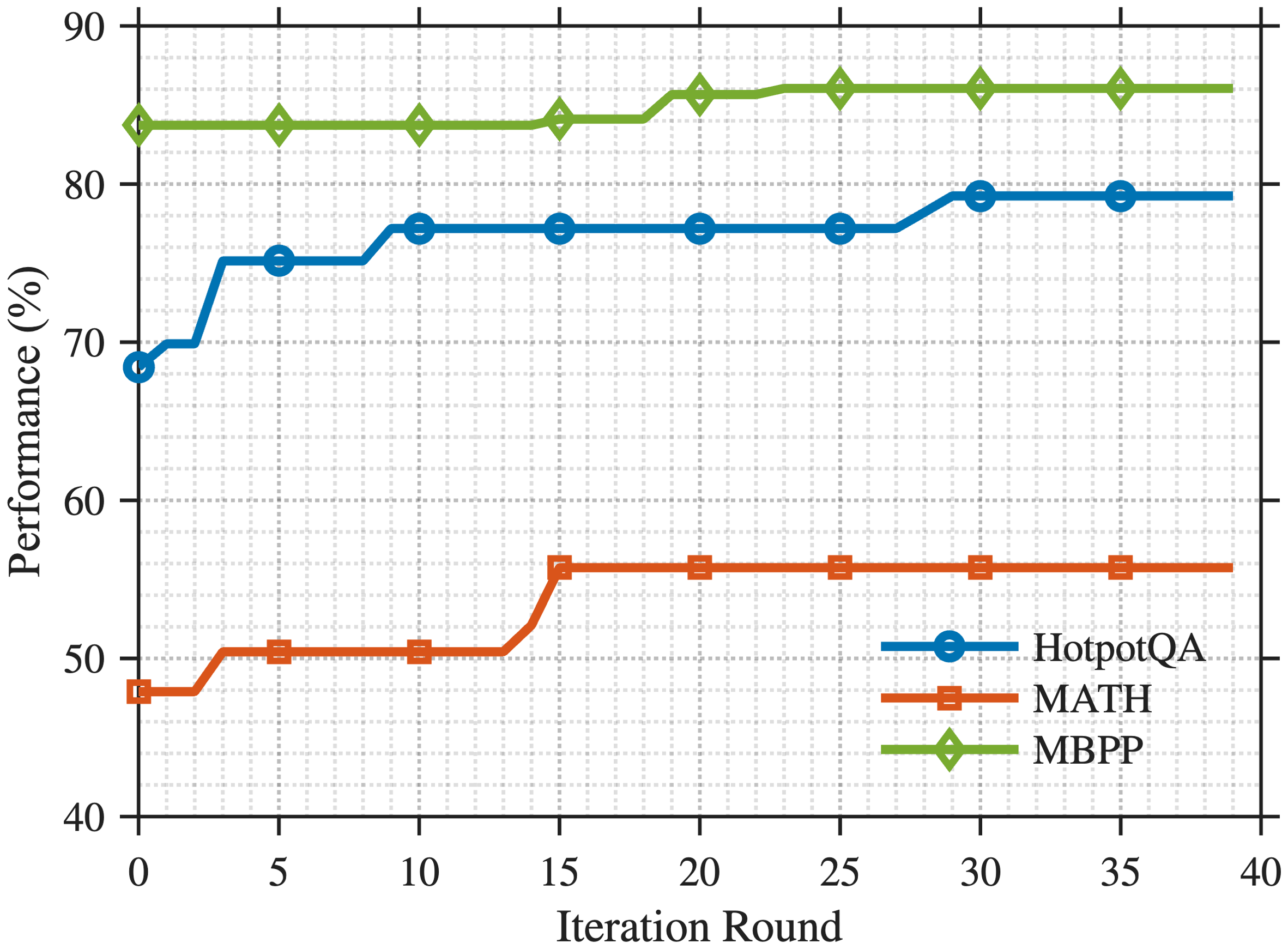}
\captionsetup{font=footnotesize,justification=raggedright,singlelinecheck=false}
\caption{\textbf{Validation best-so-far trajectory during optimization.}
Running best validation scores on HotpotQA, MATH, and MBPP.}
\label{fig:learning}
% AUTHOR CHECK: The supplied plot extends to about 40 rounds, whereas T=30
% counts fusion iterations. Clarify whether the axis includes initialization.
\end{wrapfigure}

\paragraph{Optimization Trajectory.}
\cref{fig:learning} reports the \emph{validation} best-so-far performance during optimization.
Across all three tasks, DAGO exhibits a monotonic improvement trend in the running best solution and typically plateaus by the end of the budget.
This pattern indicates that bandit-guided multi-parent fusion steadily discovers and retains higher-quality workflow structures, rather than relying on occasional gains from unguided perturbations.

\paragraph{Zero-Shot Cross-Task Transfer.}

\begin{figure}[t]
\centering
\captionsetup{font=footnotesize}
\begin{minipage}[t]{0.48\textwidth}
\vspace{0pt}
\centering
\includegraphics[width=\linewidth]{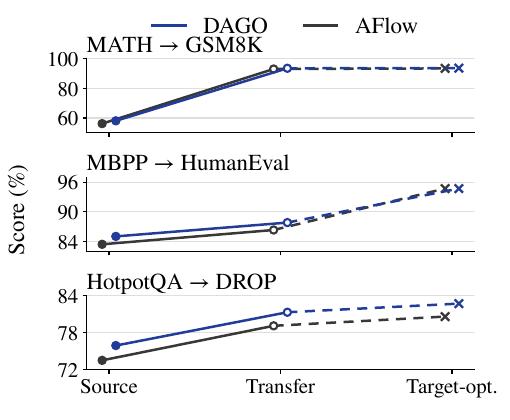}
\caption{\textbf{Cross-task transfer trajectories.} Source / Transfer: source-optimized workflow evaluated on source / target tasks. Target-opt.: target-specific search with the same budget. Transfer uses no re-optimization.}
\label{fig:transfer_vis}
\end{minipage}\hfill
\begin{minipage}[t]{0.49\textwidth}
\vspace{0pt}
\captionof{table}{\textbf{Cross-task transfer.} Means over three test executions: target-optimized $\mathrm{Opt}_t$, zero-shot $\mathrm{Trans}_{s\to t}$, regret $R=\mathrm{Opt}_t-\mathrm{Trans}_{s\to t}$, and normalized regret $\widetilde{R}=R/\mathrm{Opt}_t$. Avg. covers three pairs.}
\label{tab:transfer}
\centering
\small
\setlength{\tabcolsep}{2pt}
\renewcommand{\arraystretch}{0.9}
\begin{tabular}{@{}lcccc@{}}
\toprule
Method & $\mathrm{Opt}_t\uparrow$ & $\mathrm{Trans}_{s\to t}\uparrow$ & $R_{s\to t}\downarrow$ & $\widetilde{R}_{s\to t}\downarrow$ \\
\midrule
\multicolumn{5}{l}{\textit{MATH $\to$ GSM8K}} \\
AFlow & 93.5 & 93.1 & 0.4 & 0.0043 \\
DAGO & \textbf{93.7} & \textbf{93.6} & \textbf{0.1} & \textbf{0.0011} \\
\midrule
\multicolumn{5}{l}{\textit{MBPP $\to$ HumanEval}} \\
AFlow & \textbf{94.7} & 86.3 & 8.4 & 0.0887 \\
DAGO & \textbf{94.7} & \textbf{87.8} & \textbf{6.9} & \textbf{0.0729} \\
\midrule
\multicolumn{5}{l}{\textit{HotpotQA $\to$ DROP}} \\
AFlow & 80.6 & 79.1 & 1.5 & 0.0186 \\
DAGO & \textbf{82.7} & \textbf{81.3} & \textbf{1.4} & \textbf{0.0169} \\
\midrule
Avg.\ (AFlow) & --- & --- & 3.43 & 0.0372 \\
Avg.\ (DAGO) & --- & --- & \textbf{2.80} & \textbf{0.0303} \\
\bottomrule
\end{tabular}
\end{minipage}
\end{figure}

\cref{fig:transfer_vis,tab:transfer} evaluate whether the best workflows discovered on a source task generalize to a related target task \emph{without} re-optimization.
Across all three transfer pairs, DAGO yields consistently lower transfer regret (both absolute and normalized), suggesting that multi-parent fusion tends to distill reusable workflow components that are less tied to source-task idiosyncrasies.
\WFclear
\begin{figure}[t]
\centering
\captionsetup{font=footnotesize,justification=raggedright,singlelinecheck=false}
\begin{minipage}[t]{0.25\textwidth}
\vspace{0pt}
\captionof{table}{\textbf{Search cost (USD).} Total covers six benchmarks; lower costs are bold.}
\label{tab:search_cost}
\end{minipage}\hfill
\begin{minipage}[t]{0.40\textwidth}
\vspace{0pt}
\captionof{table}{\textbf{Inference performance and cost.} Means over three test executions. Overall / Cost: three-task average / total USD. Hotpot: HotpotQA. Better results per executor are bold.}
\label{tab:cost}
\end{minipage}\hfill
\begin{minipage}[t]{0.31\textwidth}
\vspace{0pt}
\captionof{figure}{\textbf{Inference cost--performance.} Table~\ref{tab:cost} with log-scale cost; dashed line: Pareto frontier.}
\label{fig:pareto}
\end{minipage}
\par
\begin{minipage}[t]{0.25\textwidth}
\vspace{0pt}
\centering
\footnotesize
\setlength{\tabcolsep}{1.5pt}
\renewcommand{\arraystretch}{1.02}
\begin{tabular}{@{}lrr@{}}
\toprule
Task & AFlow & DAGO \\
\midrule
HotpotQA & \textbf{13.39} & 18.89 \\
DROP & \textbf{4.85} & 5.64 \\
HumanEval & 2.63 & \textbf{1.92} \\
MBPP & 9.30 & \textbf{6.93} \\
GSM8K & 20.52 & \textbf{18.18} \\
MATH & 35.76 & \textbf{25.22} \\
\midrule
Total & 86.45 & \textbf{76.78} \\
\bottomrule
\end{tabular}
\end{minipage}\hfill
\begin{minipage}[t]{0.40\textwidth}
\vspace{0pt}
\centering
\fontsize{7.5}{9}\selectfont
\setlength{\tabcolsep}{1.5pt}
\renewcommand{\arraystretch}{0.86}
\begin{tabular}{@{}lccccc@{}}
\toprule
Method & MATH & MBPP & Hotpot & Overall & Cost \\
\midrule
\multicolumn{6}{l}{\textit{GPT-4o-mini}~\cite{gpt4o}} \\
AFlow & 56.2 & 83.4 & 73.5 & 71.0 & 2.7278 \\
DAGO & \textbf{58.1} & \textbf{85.0} & \textbf{75.9} & \textbf{73.0} & \textbf{2.4344} \\
\midrule
\multicolumn{6}{l}{\textit{DeepSeek-V3.2}~\cite{liu2025deepseek}} \\
AFlow & 85.8 & 87.1 & 78.1 & 83.7 & 4.0197 \\
DAGO & \textbf{86.3} & \textbf{87.7} & \textbf{80.5} & \textbf{84.8} & \textbf{3.7368} \\
\midrule
\multicolumn{6}{l}{\textit{Gemini-3-Flash}~\cite{ge3}} \\
AFlow & \textbf{87.0} & 92.1 & 79.4 & 86.2 & 14.5263 \\
DAGO & 86.8 & \textbf{93.0} & \textbf{81.6} & \textbf{87.1} & \textbf{10.5712} \\
\bottomrule
\end{tabular}
\end{minipage}\hfill
\begin{minipage}[t]{0.31\textwidth}
\vspace{0pt}
\centering
\includegraphics[width=\linewidth]{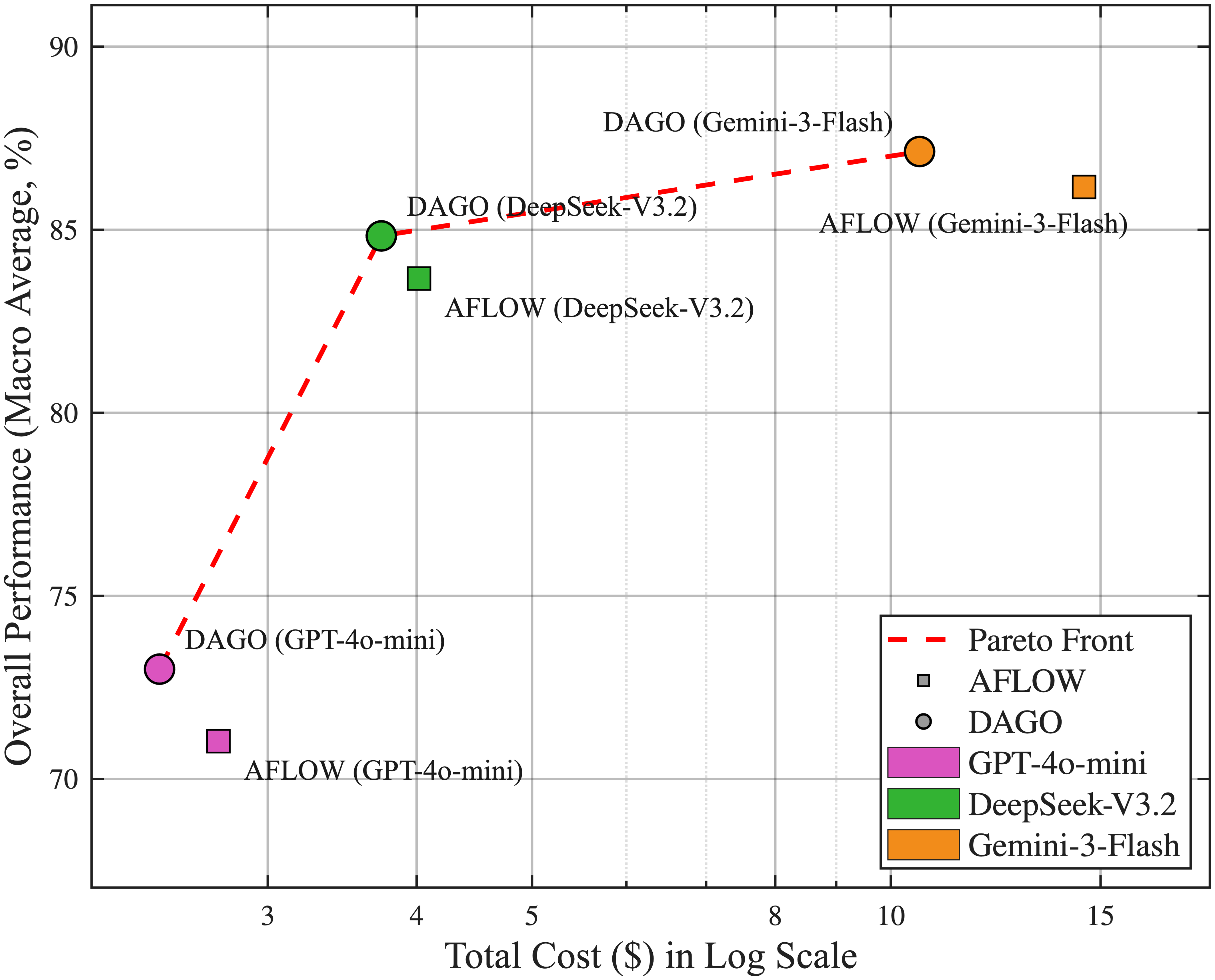}
\end{minipage}
\end{figure}

\paragraph{Search Cost and Inference Cost--Performance Trade-offs.}
\cref{tab:search_cost} separates search expenditure from inference cost. DAGO reduces total search cost from \$86.45 (AFlow) to \$76.78, an $11.2\%$ saving overall. It is cheaper on four benchmarks, with the largest saving on MATH (\$10.54), but more expensive on HotpotQA and DROP, so the reduction is aggregate rather than uniform. For inference, \cref{tab:cost,fig:pareto} show that DAGO achieves a higher three-task average on MATH, MBPP, and HotpotQA at lower cost for all execution models: costs decrease by $10.8\%$, $7.0\%$, and $27.2\%$ with GPT-4o-mini, DeepSeek-V3.2, and Gemini-3-Flash, while average scores increase from $71.0$ to $73.0$, $83.7$ to $84.8$, and $86.2$ to $87.1$, respectively. These inference comparisons exclude search expenditure and cover only the three reported test sets.

%==============================================================================

\subsection{Ablation Study}
\label{sec:ablation}

\paragraph{Effect of Bandit-Guided Arm Selection.}

\cref{tab:ab_selection} compares LinUCB with random selection and different exploration strengths.
Random selection lowers the score by $2.9$ points on GSM8K, $3.1$ on HumanEval, and $1.9$ on DROP relative to the default $\nu{=}0.1$.
Both greedy selection ($\nu{=}0$) and stronger exploration ($\nu{=}1$) also yield lower means on all three tasks.
Among the tested settings, moderate exploration gives the best results; neither random nor greedy arm selection matches it.

\paragraph{Effect of Individual Components.}
\cref{tab:ab_components} evaluates four components on GSM8K, HotpotQA, and MBPP. Removing parent summaries or replacing UCB with random selection reduces the three-task average from $84.9$ to $82.9$ and $83.0$, respectively, while single-parent fusion yields $83.3$ and incurs the largest drop on HotpotQA ($2.1$ points), supporting the use of parent summaries, multi-parent fusion, and bandit-guided selection. Full-matrix LinUCB achieves $84.2$, below the diagonal variant's $84.9$; with high-dimensional arm features and at most $30$ fusion rewards per search, estimating cross-feature covariances may be less reliable, while the diagonal approximation also avoids maintaining a full $Jd\times Jd$ covariance matrix.

\begin{figure}[!htbp]
\centering
\captionsetup{font=footnotesize,justification=raggedright,singlelinecheck=false}
\begin{minipage}[t]{0.42\textwidth}
\vspace{0pt}
\captionof{figure}{\textbf{Effect of arm size $J$.} Means over three test executions with varying numbers of parents.}
\label{fig:ab_nodes}
\end{minipage}\hfill
\begin{minipage}[t]{0.54\textwidth}
\vspace{0pt}
\captionof{table}{\textbf{Robustness across execution models.} Means over three test executions for the selected workflows.}
\label{tab:ab_models}
\end{minipage}
\par
\begin{minipage}[t]{0.42\textwidth}
\vspace{0pt}
\centering
\includegraphics[width=0.75\linewidth]{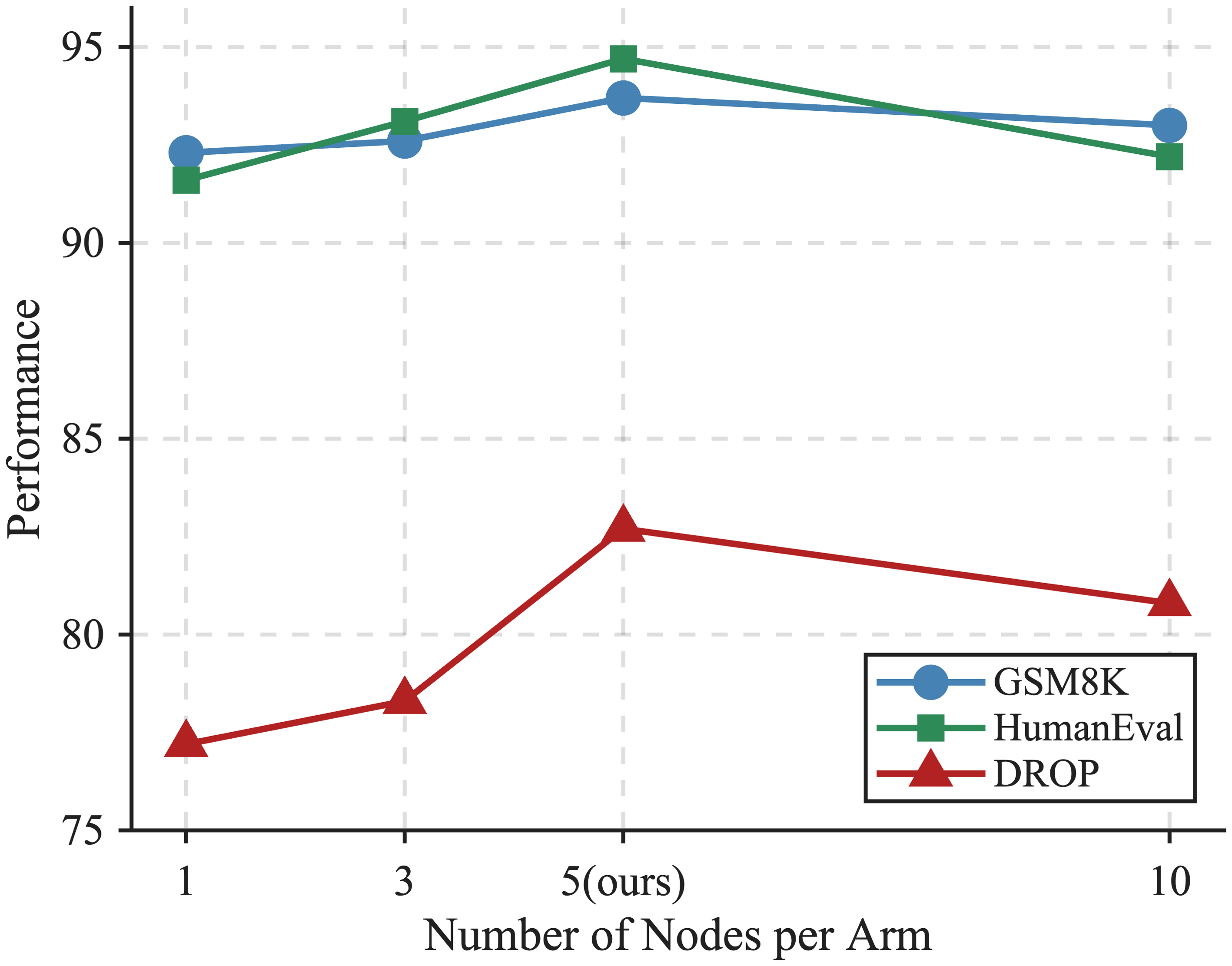}
\end{minipage}\hfill
\begin{minipage}[t]{0.54\textwidth}
\vspace{0pt}
\centering
\small
\setlength{\tabcolsep}{2pt}
\renewcommand{\arraystretch}{1.05}
\begin{tabular*}{\linewidth}{@{\extracolsep{\fill}}llccc@{}}
\toprule
Model & Method & MATH & MBPP & HotpotQA \\
\midrule
\multirow{2}{*}{GPT-4o-mini} & AFlow & 56.2 & 83.4 & 73.5 \\
& DAGO & \textbf{58.1} & \textbf{85.0} & \textbf{75.9} \\
\midrule
\multirow{2}{*}{DeepSeek-V3.2} & AFlow & 85.8 & 87.1 & 78.1 \\
& DAGO & \textbf{86.3} & \textbf{87.7} & \textbf{80.5} \\
\midrule
\multirow{2}{*}{Gemini-3-Flash} & AFlow & \textbf{87.0} & 92.1 & 79.4 \\
& DAGO & 86.8 & \textbf{93.0} & \textbf{81.6} \\
\bottomrule
\end{tabular*}
\end{minipage}
\end{figure}

\begin{table}[t]
\centering
\begin{minipage}{\textwidth}
\captionsetup{font=footnotesize,justification=raggedright,singlelinecheck=false}
\begin{minipage}[t]{0.44\textwidth}
\vspace{0pt}
\caption{\textbf{Arm selection.} Means over three test executions; $\nu$ controls UCB exploration.}
\label{tab:ab_selection}
\end{minipage}\hfill
\begin{minipage}[t]{0.54\textwidth}
\vspace{0pt}
\caption{\textbf{Component ablations.} Means over three test executions; Avg. is the three-task macro average. Best scores are bold.}
\label{tab:ab_components}
\end{minipage}
\par
\begin{minipage}[t]{0.44\textwidth}
\vspace{0pt}
\centering
\fontsize{7.5}{9}\selectfont
\setlength{\tabcolsep}{2pt}
\renewcommand{\arraystretch}{1.2}
\begin{tabular*}{\linewidth}{@{\extracolsep{\fill}}lccc@{}}
\toprule
Method & GSM8K & HumanEval & DROP \\
\midrule
AFlow & 93.5 & \textbf{94.7} & 80.6 \\
DAGO ($\nu{=}0.1$) & \textbf{93.7} & \textbf{94.7} & \textbf{82.7} \\
\midrule
DAGO (Random) & 90.8 & 91.6 & 80.8 \\
DAGO ($\nu{=}0$) & 92.0 & 93.9 & 81.4 \\
DAGO ($\nu{=}1$) & 91.7 & 92.4 & 82.4 \\
\bottomrule
\end{tabular*}
\end{minipage}\hfill
\begin{minipage}[t]{0.54\textwidth}
\vspace{0pt}
\centering
\fontsize{7.5}{9}\selectfont
\setlength{\tabcolsep}{3pt}
\renewcommand{\arraystretch}{1.03}
\begin{tabular*}{\linewidth}{@{\extracolsep{\fill}}lcccc@{}}
\toprule
Setting & GSM8K & HotpotQA & MBPP & Avg. \\
\midrule
AFlow & 93.5 & 73.5 & 83.4 & 83.5 \\
DAGO (full) & \textbf{93.7} & \textbf{75.9} & \textbf{85.0} & \textbf{84.9} \\
\midrule
w/o summarization & 92.0 & 74.2 & 82.6 & 82.9 \\
w/o UCB (Random) & 90.8 & 74.5 & 83.7 & 83.0 \\
Full-matrix LinUCB & 93.2 & 75.1 & 84.2 & 84.2 \\
Single-parent only & 92.7 & 73.8 & 83.5 & 83.3 \\
\bottomrule
\end{tabular*}
\end{minipage}
\end{minipage}
\end{table}

\paragraph{Effect of Arm Size.}
\cref{fig:ab_nodes} studies the impact of arm size, i.e., the number of parent workflows contributing to each fusion step.
When $J{=}1$, DAGO degenerates to single-parent evolution and performs worst, suggesting limited structural diversity.
Performance improves as $J$ increases and peaks around $J{=}5$, while a larger arm ($J{=}10$) leads to degradation, consistent with a trade-off:
increasing $J$ can enrich complementary signals, but overly many parents may introduce incompatible patterns and dilute the fusion objective.

\paragraph{Robustness to Execution Model Choice.}
\cref{tab:ab_models} indicates that DAGO generally maintains its advantage across execution models, improving on two out of three tasks even when the executor is already strong.
This supports the interpretation that DAGO's gains primarily stem from discovering better workflow structures (e.g., decomposition, verification, and coordination patterns), rather than relying on idiosyncrasies of a specific base model.

\section{Conclusion}
\label{sec:conclusion}

We presented DAGO, a framework for LLM-based workflow optimization that learns which parent workflows to fuse from the validation feedback of their offspring.
By treating parent combinations as bandit arms, DAGO uses a shared reward predictor and a diagonal UCB-style policy to guide selection, while a DAG archive and summary-guided fusion support the reuse and refinement of discovered workflows.
Across six benchmarks spanning mathematical reasoning, code generation, and question answering, DAGO achieves the highest macro-average score among the evaluated methods.
Under the same validation-evaluation budget cap as AFlow, it improves the macro-average score from $80.3$ to $81.7$ while reducing aggregate search expenditure by $11.2\%$.
Ablations further support the value of feedback-guided selection and exploration, alongside the benefits of multi-parent fusion.
Together, these results highlight feedback-guided parent selection as a key component of effective workflow fusion: prior search outcomes inform not only which workflows are available for reuse, but also which combinations to explore next.

\section*{AI Use Statement}

In this work, we used generative AI tools to improve the clarity, grammar, and presentation of the manuscript. We also used these tools to assist with implementing, debugging, and refining parts of the research code. All AI-assisted text was reviewed and revised by the authors, and all AI-assisted code was inspected and tested by the authors. The authors take full responsibility for the final manuscript, implementation, and experimental results.

\section*{Ethics Statement}

This work presents a methodological study of automated agentic workflow optimization using existing language models and standard benchmarks for mathematical reasoning, code generation, and question answering. The study does not involve human participants, private or sensitive personal data, or deployment in high-stakes settings. Our experiments evaluate workflows within established benchmark protocols. We are not aware of any specific ethical concerns arising from the experiments conducted in this study.

\section*{Reproducibility Statement}

The main text and appendices document the DAGO optimizer, prompt templates, task operators, and evaluation protocol, together with experimental configurations and concrete workflow examples.\newline Section~\ref{sec:method} and Algorithm~\ref{alg:dago} describe the DAGO optimization procedure, including candidate-arm proposal, bandit-guided parent selection, and workflow fusion. Appendix~\ref{app:theory} provides the theoretical analysis and its assumptions. Section~\ref{sec:experiments} specifies the models, benchmarks, baselines, evaluation budgets, and reporting protocol. Appendices~\ref{app:hyperparameters} and~\ref{app:datasets} detail the hyperparameters, stopping criteria, dataset splits, and evaluation metrics. Appendix~\ref{app:prompts} presents the LLM prompt templates, and Appendix~\ref{app:algorithm} provides extended pseudocode for sampling, diagonal LinUCB updates, embedding construction, and the full optimization loop. Appendices~\ref{app:workflows} and~\ref{app:results} document discovered workflows, their code and prompts, and detailed experimental results. These descriptions specify the search and evaluation settings used in our experiments and provide concrete workflow examples for reimplementation.

\bibliography{main}
\bibliographystyle{iclr2027/iclr2027_conference}

%%%%%%%%%%%%%%%%%%%%%%%%%%%%%%%%%%%%%%%%%%%%%%%%%%%%%%%%%%%%%%%%%%%%%%%%%%%%%%%
% APPENDIX
%%%%%%%%%%%%%%%%%%%%%%%%%%%%%%%%%%%%%%%%%%%%%%%%%%%%%%%%%%%%%%%%%%%%%%%%%%%%%%%
\newpage
\appendix
\onecolumn
\raggedbottom
\setlength{\parskip}{4pt}

%------------------------------------------------------------------------------
% Appendix Table of Contents
%------------------------------------------------------------------------------
\vfill
\begin{center}
{\LARGE\bfseries Appendix}
\end{center}

\makeatletter
\newcommand{\appendixtocentry}[3]{%
    \noindent\normalsize\hyperref[#1]{\textbf{#2}\hspace{0.5em}#3}\dotfill\pageref{#1}\par\vspace{1.5em}%
}
\makeatother

\appendixtocentry{app:theory}{Appendix A}{Theoretical Perspective on Bandit-Guided Multi-Parent Fusion}

\appendixtocentry{app:hyperparameters}{Appendix B}{Hyperparameter Settings}

\appendixtocentry{app:datasets}{Appendix C}{Dataset Details}

\appendixtocentry{app:prompts}{Appendix D}{LLM Prompt Templates}

\appendixtocentry{app:algorithm}{Appendix E}{Extended Algorithm Details}

\appendixtocentry{app:workflows}{Appendix F}{Comprehensive Analysis of Discovered Workflows}

\appendixtocentry{app:results}{Appendix G}{Full Experimental Results}

\vfill
\newpage

%==============================================================================
\section{Decision Analysis without Linear Realizability}
\label{app:theory}

This appendix characterizes the implemented selection rule without assuming that LLM-driven fusion rewards are linear in pretrained embeddings.
Its results separate what follows algebraically from the rule, what requires accurate predictions, and what remains an empirical question.
They are not a new general-purpose bandit algorithm or a regret guarantee for the complete workflow-generation process.

\subsection{History, potential rewards, and comparison scope}
\label{app:analysis_scope}

Immediately before round $t$, the archive is $\mathcal{G}_{t-1}$.
Let $\bar{\mathcal{A}}_t$ contain all canonically ordered arms of $J$ distinct parents from that archive.
The proposal mechanism repeatedly samples arms and rejects duplicates until the candidate set $\mathcal{A}_t$ contains $M_t=\min\{M,|\bar{\mathcal{A}}_t|\}$ distinct arms, where $M$ is the candidate cap.
Let $\mathcal{H}_t$ include previous decisions and rewards, the archive, cached features, fixed model configurations, the current bandit state, and the realized candidates, before new fusion randomness is drawn.
For every feasible arm $a$, define the potential reward
\begin{equation}
 r_t(a)=\mathrm{Eval}_{q_t}\bigl(\mathrm{Fuse}(a;\xi_t^{\mathrm{gen}});\xi_t^{\mathrm{eval}}\bigr),
 \qquad \mu_t(a)=\mathbb{E}[r_t(a)\mid\mathcal{H}_t]\in[0,1].
 \label{eq:potential_rewards}
\end{equation}
Here fusion includes summarization and the full parent inputs; the random variables include any remaining generation and execution variability.
The observed reward is $r_t=r_t(a_t)$.
Averaging several executions of the \emph{same generated child} reduces evaluation variability only to the extent permitted by their dependence; it does not average over independently generated children.
No independence between executions, sub-Gaussian parameter, stationary reward law, or realizable linear function is needed for the pathwise inequalities below.

Write $m_t(a)=\widehat\mu_t(a)$, $h_t(a)\ge0$, and $U_t(a)=m_t(a)+\nu h_t(a)$ for the actual diagonal scores of Eq.~\eqref{eq:ucb}.
The notation $m_t$ avoids conflating a fitted score with the conditional mean $\mu_t$.
The fitted score need not lie in $[0,1]$.
Throughout, $a_t$ maximizes $U_t$ over the proposed candidates, with a fixed tie-breaking rule.

\paragraph{A one-step comparator, not an oracle search trajectory.}
The full-arm comparator uses the archive \emph{actually reached by DAGO} at that round.
An optimizer making different earlier selections would generally create a different archive and face different arms.
Consequently, the analysis does not compare DAGO with the best possible sequence of workflow-generation decisions or with the globally optimal program.
Expanding the archive changes the available arms; it does not, by itself, imply that the conditional reward of a fixed arm changes.
The notation $\mu_t$ nevertheless allows changes in the reward law.

\subsection{Separating candidate proposal from arm selection}
\label{app:proposal_gap}

Let
\begin{equation}
 \bar a_t^\star\in\arg\max_{a\in\bar{\mathcal{A}}_t}\mu_t(a),
 \qquad
 a_t^\star\in\arg\max_{a\in\mathcal{A}_t}\mu_t(a).
\end{equation}
Define the proposal gap $g_t=\mu_t(\bar a_t^\star)-\mu_t(a_t^\star)\ge0$ and selection gap
$\Delta_t=\mu_t(a_t^\star)-\mu_t(a_t)\ge0$.
For any realized run of $K$ completed rounds,
\begin{equation}
 R_K^{\mathrm{archive}}
 :=\sum_{t=1}^{K}\bigl(\mu_t(\bar a_t^\star)-\mu_t(a_t)\bigr)
 =\underbrace{\sum_{t=1}^{K}g_t}_{\text{proposal loss}}
 +\underbrace{\sum_{t=1}^{K}\Delta_t}_{R_K^{\mathrm{sel}}}.
 \label{eq:proposal_selection_identity}
\end{equation}
This is an exact identity for arbitrary reward functions.
No selector restricted to $\mathcal{A}_t$ can eliminate loss caused by missing a useful arm during proposal.

\begin{proposition}[Coverage with rejection sampling to distinct arms]
\label{prop:proposal_coverage}
Condition on the history before proposal and on the round's full-arm reward function.
Let $F_t=|\bar{\mathcal{A}}_t|$, $M_t=\min\{M,F_t\}$, and let $Q_t$ be a fixed distribution with positive mass on every feasible arm.
Draw independent \emph{raw proposals} from $Q_t$, rejecting repetitions until $M_t$ distinct arms have been accepted.
For $\gamma\ge0$, define
\begin{equation}
 \mathcal{S}_{t,\gamma}
 =\{a:\mu_t(a)\ge \max_{b\in\bar{\mathcal{A}}_t}\mu_t(b)-\gamma\},
 \qquad q_{t,\gamma}=Q_t(\mathcal{S}_{t,\gamma}).
\end{equation}
If proposal randomization does not change the reward distribution of a fixed arm, then
\begin{equation}
 \Pr(g_t>\gamma)\le(1-q_{t,\gamma})^{M_t}
 \le e^{-M_t q_{t,\gamma}},
 \qquad
 \mathbb{E}[g_t]\le\gamma+e^{-M_t q_{t,\gamma}}.
 \label{eq:proposal_coverage_bound}
\end{equation}
The first relation is an inequality, not the equality for $M_t$ independent draws with replacement.
\end{proposition}
\begin{proof}
Every feasible arm has positive probability, so collecting $M_t\le F_t$ distinct arms terminates almost surely.
Couple this procedure with the same infinite sequence of independent raw proposals and let $L_t\ge M_t$ be the number consumed before termination.
Every arm appearing in the first $M_t$ raw proposals belongs to the final accepted set: it is either accepted on first appearance or was accepted earlier.
Therefore, if the final candidate set misses $\mathcal{S}_{t,\gamma}$, the first $M_t$ raw proposals must also miss it.
The latter event has probability $(1-q_{t,\gamma})^{M_t}$, proving the first inequality.
The inequality $1-u\le e^{-u}$ gives the exponential bound.
Finally, $0\le g_t\le1$ implies $\mathbb{E}[g_t]\le\gamma+\Pr(g_t>\gamma)$.
\end{proof}

\paragraph{Accepted arms are dependent.}
Let $\mathcal{U}_{m-1}$ be the first $m-1$ accepted arms.
For $a\notin\mathcal{U}_{m-1}$, rejection sampling gives
\begin{equation}
 \Pr(A_m=a\mid\mathcal{U}_{m-1})
 =\frac{Q_t(a)}{1-Q_t(\mathcal{U}_{m-1})},
 \qquad Q_t(\mathcal{U})=\sum_{b\in\mathcal{U}}Q_t(b).
 \label{eq:unique_arm_conditional}
\end{equation}
Previously accepted arms have probability zero.
Thus accepted candidates are neither independent nor, in general, uniformly sampled.
For example, if $Q_t$ is uniform and $s=|\mathcal{S}_{t,\gamma}|$, the exact miss probability is $\binom{F_t-s}{M_t}/\binom{F_t}{M_t}$, with value zero when $F_t-s<M_t$.
In particular, collecting the entire finite arm space gives $g_t=0$.

The reward-invariance condition concerns the specified fusion interface, which receives only the selected parents and their summaries, not the other candidates; it also requires proposal draws not to alter the conditional fusion-randomness law.
The coupling comparison holds for the same $Q_t$ and the same target of $M_t$ candidates, not for the same number of raw sampling attempts or the same wall-clock proposal cost.
Rejection sampling may require substantially more than $M_t$ raw proposals, although these attempts consume no validation evaluations.
Moreover, positive support alone does not make $q_{t,\gamma}$ large or ensure a useful numerical bound.
No realizable linear reward law is needed for this coverage result.

\subsection{A decision bound for an arbitrary nonlinear reward function}
\label{app:nonlinear_decision_bound}

Define the actual prediction residual and its span over the current candidates:
\begin{equation}
 e_t(a)=\mu_t(a)-m_t(a),
 \qquad
 \varepsilon_t=\max_{a\in\mathcal{A}_t}e_t(a)-\min_{a\in\mathcal{A}_t}e_t(a).
 \label{eq:residual_span}
\end{equation}
This is a property of the fitted score on the current candidates, not an assumption that the true reward has a small linear approximation error.
The span measures errors in \emph{predicted reward differences}: it is unchanged by adding the same constant to all predictions.
If $|e_t(a)|\le\epsilon_t$ happens to hold for all candidates, then $\varepsilon_t\le2\epsilon_t$, but no such calibration assumption is imposed.

\begin{proposition}[Selection error of the implemented score]
\label{prop:selection_error}
For any nonempty candidate collection, any real-valued predictor $m_t$, any nonnegative exploration proxy $h_t$, any $\nu\ge0$, and any reward function $\mu_t:\mathcal{A}_t\to[0,1]$, the maximizer of $m_t+\nu h_t$ satisfies
\begin{equation}
 0\le\Delta_t
 \le e_t(a_t^\star)-e_t(a_t)
      +\nu\bigl(h_t(a_t)-h_t(a_t^\star)\bigr)
 \le\varepsilon_t+\nu h_t(a_t).
 \label{eq:exact_score_decision_bound}
\end{equation}
\end{proposition}
\begin{proof}
Since $U_t(a_t)\ge U_t(a_t^\star)$,
$m_t(a_t^\star)-m_t(a_t)\le\nu(h_t(a_t)-h_t(a_t^\star))$.
Adding $e_t(a_t^\star)-e_t(a_t)$ to both sides gives the first upper bound.
The residual difference is at most its span and $h_t(a_t^\star)\ge0$, giving the second bound.
\end{proof}

\paragraph{What the bound explains.}
A shared predictor is useful when its reward differences help discriminate promising arms, even if its absolute predictions are biased.
Exploration can deliberately choose an arm with lower predicted reward to collect feedback from less-covered features.
Its immediate decision cost is visible in Eq.~\eqref{eq:exact_score_decision_bound}; any subsequent benefit must appear through better future selection or predictions.
The proposition does \emph{not} prove that exploration reduces future residual spans or that embeddings encode all relevant parent interactions.

\paragraph{A ranking-margin interpretation.}
For any suboptimal candidate $a$, it cannot be selected whenever
\begin{equation}
 \mu_t(a_t^\star)-\mu_t(a)
 >e_t(a_t^\star)-e_t(a)
   +\nu\bigl(h_t(a)-h_t(a_t^\star)\bigr),
 \label{eq:ranking_margin}
\end{equation}
since then $U_t(a_t^\star)>U_t(a)$.
This relates a selection decision to its prediction-difference error and exploration preference without attributing linear structure to the fusion operator.

\subsection{Accumulated diagonal exploration bonuses}
\label{app:diagonal_potential}

The following calculation concerns the \emph{implemented} statistic, rather than a full-matrix confidence ellipsoid.
It requires a finite feature-norm bound and unchanged additive diagonal updates; it does not require any reward model.
For a finite run one may use $L=\max_{1\le t\le K}\|\mathbf{x}_t\|_2$ in the statement; claiming an asymptotic rate would require a uniform bound as the horizon grows.
No extra feature normalization is introduced into DAGO by this analysis.

\begin{proposition}[A pathwise bound on diagonal bonuses]
\label{prop:diagonal_bonus_sum}
Suppose $d_{0,k}=\lambda>0$, $d_{t,k}=d_{t-1,k}+x_{t,k}^2$, and $\|\mathbf{x}_t\|_2\le L$.
Let $h_t(a_t)^2=\sum_k x_{t,k}^2/d_{t-1,k}$ and $c_{\lambda,L}=1+L^2/\lambda$.
Then
\begin{equation}
 \sum_{t=1}^{K}h_t(a_t)
 \le \sqrt{Kc_{\lambda,L}\sum_{k=1}^{p}
          \log\!\left(1+\frac{\sum_{t=1}^{K}x_{t,k}^2}{\lambda}\right)}
 \le\sqrt{Kc_{\lambda,L}p\log\!\left(1+\frac{KL^2}{p\lambda}\right)}.
 \label{eq:diagonal_bonus_sum}
\end{equation}
\end{proposition}
\begin{proof}
Set $u_{t,k}=x_{t,k}^2/d_{t-1,k}$.
Because $d_{t-1,k}\ge\lambda$, $0\le u_{t,k}\le L^2/\lambda$.
The elementary inequality $\log(1+u)\ge u/(1+u)$ gives
$u_{t,k}\le c_{\lambda,L}\log(1+u_{t,k})$.
Therefore,
\begin{align}
 \sum_{t=1}^{K}h_t(a_t)^2
 &=\sum_{t=1}^{K}\sum_{k=1}^{p}u_{t,k}\\
 &\le c_{\lambda,L}\sum_{k=1}^{p}\sum_{t=1}^{K}
       \log\frac{d_{t,k}}{d_{t-1,k}}
 =c_{\lambda,L}\sum_{k=1}^{p}\log\frac{d_{K,k}}{\lambda}.
\end{align}
Cauchy--Schwarz yields the first inequality in Eq.~\eqref{eq:diagonal_bonus_sum}.
Concavity of the logarithm and $\sum_{t,k}x_{t,k}^2\le KL^2$ yield the second.
\end{proof}

\begin{corollary}[Error-dependent archive-relative decision bound]
\label{cor:archive_decision_bound}
Under the update and feature conditions of Proposition~\ref{prop:diagonal_bonus_sum},
\begin{equation}
 R_K^{\mathrm{archive}}
 \le\sum_{t=1}^{K}g_t+\sum_{t=1}^{K}\varepsilon_t
 +\nu\sqrt{Kc_{\lambda,L}p\log\!\left(1+\frac{KL^2}{p\lambda}\right)}.
 \label{eq:archive_decision_bound}
\end{equation}
The bound also holds at an early-stopped horizon because each ingredient is pathwise.
\end{corollary}
\begin{proof}
Combine Eq.~\eqref{eq:proposal_selection_identity}, Proposition~\ref{prop:selection_error}, and Proposition~\ref{prop:diagonal_bonus_sum}.
\end{proof}

This is an \emph{error-dependent diagnostic}, not an assertion of sublinear regret.
Only the accumulated exploration proxy is bounded directly by the chosen update.
Neither $\sum_tg_t$ nor $\sum_t\varepsilon_t$ is shown to be sublinear, and either can grow linearly.
For fixed $p,\lambda,L,\nu$, sublinear growth of both terms would be sufficient for vanishing average archive-relative gaps; the reported workflow experiments do not establish those premises.
Moreover, with $p=Jd=12{,}800$ and at most $30$ fusion decisions in the reported configuration, the numerical bound can be vacuous relative to the elementary bound $R_K^{\mathrm{archive}}\le K$.
It should not be presented as a quantitative explanation of the observed gains.

\subsection{Why standard linear-bandit guarantees do not transfer automatically}
\label{app:reference_vs_implemented}

Classical linear-bandit confidence arguments rely on a realizable reward model, appropriately controlled noise and feature norms, a full Gram matrix, and a compatible confidence radius~\citep{NIPS2011_e1d5be1c}.
Misspecification is not innocuous: existing work demonstrates failure modes and feature-dimension-dependent approximation effects~\citep{ghosh2017misspecified,lattimore2020goodfeatures}.
DAGO instead uses a fixed exploration coefficient and a diagonal approximation to both estimation and exploration.
Simply adding an unspecified linear-approximation term to a classical rate does not establish a theorem for this rule.

\paragraph{Both the prediction and the exploration metric change.}
For the same observed data, define
\begin{equation}
 \mathbf{G}_t=\lambda\mathbf{I}+\sum_{s=1}^{t}\mathbf{x}_s\mathbf{x}_s^\top,
 \qquad \mathbf{D}_t=\mathrm{diag}(\mathbf{G}_t),
 \qquad \mathbf{b}_t=\sum_{s=1}^{t}r_s\mathbf{x}_s.
\end{equation}
The full ridge solution is $\widehat{\boldsymbol{\theta}}_t^{\mathrm{full}}=\mathbf{G}_t^{-1}\mathbf{b}_t$.
DAGO uses $\widehat{\boldsymbol{\theta}}_t^{\mathrm{diag}}=\mathbf{D}_t^{-1}\mathbf{b}_t$.
Thus, analyzing only the difference between
$\sqrt{\mathbf{x}^\top\mathbf{G}_t^{-1}\mathbf{x}}$ and
$\sqrt{\mathbf{x}^\top\mathbf{D}_t^{-1}\mathbf{x}}$
omits a simultaneous change in the fitted mean.

\paragraph{An explicit correlated-feature example.}
Take $n$ noiseless observations with feature $(1,1)^\top$, reward $1$, and $\lambda>0$.
These observations are compatible with the exactly linear parameter $(1/2,1/2)^\top$.
Then
\begin{equation}
 \mathbf{G}_n=\begin{pmatrix}\lambda+n&n\\n&\lambda+n\end{pmatrix},
 \qquad \mathbf{D}_n=(\lambda+n)\mathbf{I},
 \qquad \mathbf{b}_n=(n,n)^\top.
\end{equation}
At the observed feature, the diagonal fitted value is $2n/(\lambda+n)\to2$, whereas the full ridge fitted value is $2n/(\lambda+2n)\to1$.
The diagonal estimate is therefore not automatically consistent even on these realizable data.
For the unobserved contrast $\mathbf{x}_{-}=(1,-1)^\top$,
\begin{equation}
 \mathbf{x}_{-}^\top\mathbf{G}_n^{-1}\mathbf{x}_{-}=2/\lambda,
 \qquad
 \mathbf{x}_{-}^\top\mathbf{D}_n^{-1}\mathbf{x}_{-}=2/(\lambda+n)\to0.
 \label{eq:diagonal_counterexample}
\end{equation}
Repeated observations of the common direction make the diagonal exploration proxy shrink even in a contrast direction whose full-matrix uncertainty has not decreased.
Conversely, along $(1,1)^\top$ the diagonal proxy exceeds the full one.
There is no universal conservative ordering between the two metrics.

\paragraph{A conditional comparison when correlations are small.}
For completeness, let $\mathbf{C}_t=\mathbf{D}_t^{-1/2}\mathbf{G}_t\mathbf{D}_t^{-1/2}-\mathbf{I}$.
If the measurable condition $\|\mathbf{C}_t\|_2\le\rho<1$ holds, spectral bounds give
\begin{equation}
 \frac{\mathbf{x}^\top\mathbf{D}_t^{-1}\mathbf{x}}{1+\rho}
 \le\mathbf{x}^\top\mathbf{G}_t^{-1}\mathbf{x}
 \le\frac{\mathbf{x}^\top\mathbf{D}_t^{-1}\mathbf{x}}{1-\rho}.
 \label{eq:conditional_metric_comparison}
\end{equation}
Also,
\begin{equation}
 |\mathbf{x}^\top(\widehat{\boldsymbol{\theta}}_t^{\mathrm{diag}}-
                    \widehat{\boldsymbol{\theta}}_t^{\mathrm{full}})|
 \le\frac{\rho}{1-\rho}\,
 \|\mathbf{D}_t^{-1/2}\mathbf{x}\|_2
 \|\mathbf{D}_t^{-1/2}\mathbf{b}_t\|_2.
 \label{eq:conditional_mean_comparison}
\end{equation}
Indeed, the eigenvalues of $\mathbf{I}+\mathbf{C}_t$ lie in $[1-\rho,1+\rho]$, and
$\|\mathbf{I}-(\mathbf{I}+\mathbf{C}_t)^{-1}\|_2\le\rho/(1-\rho)$.
These yield Eqs.~\eqref{eq:conditional_metric_comparison}--\eqref{eq:conditional_mean_comparison} by a change of coordinates and Cauchy--Schwarz.
No small-correlation condition is assumed for the reported experiments.
The calculation identifies what would need to be checked rather than restoring a missing guarantee.

\subsection{Representation limits and the final-workflow objective}
\label{app:objective_limits}

\paragraph{Additivity does not represent general parent interactions.}
Partitioning $\widehat{\boldsymbol{\theta}}$ into $J$ blocks yields
\begin{equation}
 m_t(a)=\sum_{j=1}^{J}\mathbf{e}_{(j)}^\top\widehat{\boldsymbol{\theta}}_{t-1,j}.
\end{equation}
Thus the prediction is additive across parent roles and has no explicit products between parent embeddings.
For a two-role example with scalar parent features $u,v\in\{0,1\}$, the reward $\mu(u,v)=uv$ cannot be represented by any additive model $c+\alpha u+\beta v$: the first three input pairs force $c=\alpha=\beta=0$, contradicting the reward at $(1,1)$.
This illustrates why the method is described as a useful selection surrogate, not a learned model of arbitrary structural complementarity.
Interaction features or nonlinear surrogates are possible extensions, but are not part of the reported algorithm.

\paragraph{Expected fusion reward and best observed child are different objectives.}
Consider a fixed archive with arm $a$ always producing reward $0.6$ and arm $b$ producing $1$ with probability $0.5$ and $0$ otherwise, independently across pulls.
The higher mean is $\mu(a)=0.6>0.5=\mu(b)$, so always selecting $a$ has zero cumulative regret to the mean-optimal arm.
Yet after $K\ge2$ pulls, the expected best reward from $b$ is $1-2^{-K}>0.6$.
Consequently, small selection regret does not imply optimal best-child discovery, even in a stationary two-arm example.
Returning the largest stored validation score also does not guarantee the largest true test performance, particularly with adaptive validation reuse or unequal evaluation repetition counts.

\subsection{Evidence and falsifiable diagnostics}
\label{app:diagnostics}

The existing random-arm and $\nu=0$ ablations assess whether feedback-guided selection and the exploration term help end-to-end search under the reported protocol.
They do not estimate $g_t$ or $\varepsilon_t$, prove linearity, or calibrate $h_t$ as a confidence interval.
The following diagnostics specify additional tests; \emph{no new numerical results from these tests are claimed here}.

\paragraph{Prequential prediction records.}
Save every candidate arm, its features, prediction, exploration proxy, proposal seed, selected arm, and eventual reward \emph{before} updating the model.
Evaluate predictions against subsequently observed rewards, not against the same rewards after refitting.
Selected-arm residuals are useful but do not reveal errors on unselected arms or certify the residual span.

\paragraph{Frozen-archive, multi-arm probes.}
At selected checkpoints, freeze the archive, candidate collection, model state, prompts, and execution settings.
Allocate a separately reported diagnostic budget to multiple randomly sampled candidate arms and independently generated children per arm.
Repeated executions of each child separate execution variation from generation variation.
Compare predicted differences, empirical reward differences, and the observed quality of arms selected by random, greedy, and exploration-aware rules.
Finite probes can estimate local ranking behavior, but cannot identify the full-archive optimum or uniformly bound unobserved arms.

\paragraph{Matched-budget end-to-end controls.}
Use the same initialization protocol, proposal mechanism, execution models, validation splits, and evaluation accounting for random, greedy, diagonal, and full-matrix selectors.
Report repeated-run variation as well as mean outcomes.
A full-matrix ridge control can be implemented through the low-rank identity in Appendix~\ref{app:computational_details}; it need not allocate a dense $p\times p$ matrix.
Any additional diagnostic evaluations must be separated from the search budget used for headline comparisons.

\section{Hyperparameter Settings}
\label{app:hyperparameters}
%==============================================================================

This appendix summarizes the hyperparameter configuration of DAGO used in all experiments.
Unless otherwise noted, we use the default settings in Table~\ref{tab:hyperparams} across benchmarks.
For reproducibility, we organize hyperparameters by: (i) search budget and early stopping,
(ii) candidate-arm proposal, (iii) LinUCB guidance, (iv) representations and embeddings,
(v) LLM roles and decoding, and (vi) evaluation.

\begingroup
\small
\setlength{\tabcolsep}{3.4pt}
\begin{longtable}{@{}>{\raggedright\arraybackslash}p{1.6cm}>{\raggedright\arraybackslash}p{3.1cm}>{\raggedright\arraybackslash}p{2.1cm}>{\raggedright\arraybackslash}p{6.4cm}@{}}
\caption{DAGO hyperparameters (default unless otherwise specified).}
\label{tab:hyperparams}\\
\toprule
\textbf{Category} & \textbf{Parameter} & \textbf{Value} & \textbf{Description / usage} \\
\midrule
\endfirsthead
\caption[]{DAGO hyperparameters (continued).}\\
\toprule
\textbf{Category} & \textbf{Parameter} & \textbf{Value} & \textbf{Description / usage} \\
\midrule
\endhead
\endfoot
\bottomrule
\endlastfoot
\multirow{4}{*}{Budget}
& Initial population size $N$ & 10
& Number of diverse workflows generated at initialization (inserted into the DAG before bandit-guided fusion). \\*
& Max iterations $T$ & 30
& Maximum number of fusion iterations in Algorithm~\ref{alg:dago}. \\*
& Stopping window $C_{\mathrm{patience}}$ & 5
& Check for a lack of relative improvement over the most recent five fusion-round validation scores. \\*
& Minimum iterations $T_{\min}$ & 20
& Enable early stopping only after at least $T_{\min}$ iterations. \\
\midrule
\multirow{5}{=}{\raggedright Arm proposal}
& Candidate-arm cap $M$ & 200
& Propose $M_t=\min\{M,\binom{|\mathcal{V}_{t-1}|}{J}\}$ distinct parent tuples in round $t$. \\*
& Parents per arm $J$ & 5
& Number of parent workflows fused to synthesize one child. \\*
& Temperature $\tau_{\text{init}}\rightarrow\tau_{\text{final}}$ & $2.0\rightarrow0.5$
& Linear temperature annealing for score-based softmax proposal (Appendix~\ref{app:sampling}). \\*
& Minimum probability $p_{\min}$ & 0.01
& Probability floor to prevent proposal collapse and ensure coverage over all nodes. \\*
& Within-arm duplicates & No
& Parents are sampled without replacement; $J$ is fixed during each search, with $N\ge J$. \\
\midrule
\multirow{4}{*}{LinUCB}
& Exploration coefficient $\nu$ & 0.1
& Exploration weight in Eq.~\eqref{eq:ucb}; $\nu$ is kept constant by default. \\*
& Regularization $\lambda$ & 0.2
& Ridge regularization in the bandit covariance (Eq.~\eqref{eq:ucb}). \\*
& Covariance approximation & Diagonal
& Maintain only $\mathrm{diag}(\mathbf{A})$ to compute uncertainty in $O(Jd)$ (Appendix~\ref{app:linucb}). \\*
& Reward $r$ & $[0,1]$
& Child validation score, normalized to $[0,1]$, used as the bandit reward. \\
\midrule
\multirow{3}{*}{Embeddings}
& Embedding dimension $d$ & 2560
& Dimension of the semantic embedding vector $\mathbf{e}_i$ for each workflow node. \\*
& Embedding model & Qwen3
& Pretrained code embedding model encoding \texttt{code} $\oplus$ \texttt{inner\_prompts}. \\*
& Arm feature $\mathbf{x}$ & Concat
& $\mathbf{x}=[\mathbf{e}_{m_1};\ldots;\mathbf{e}_{m_J}]\in\mathbb{R}^{Jd}$ (Appendix~\ref{app:linucb}). \\
\midrule
\multirow{3}{*}{LLMs}
& Optimization LLM & Claude-3.5-Sonnet
& Used for strategy discovery, initial workflow generation, parent summarization, and fusion. \\*
& Execution LLM & GPT-4o-mini
& Used to execute and evaluate candidate workflows during search and final reporting. \\*
& Generation temperature & 0.0
& Decoding temperature for workflow generation/fusion (summarization typically uses a lower temperature for stability). \\
\midrule
\multirow{4}{*}{Evaluation}
& Validation samples & Dataset-specific
& Fixed full validation splits: GSM8K 264; MATH 119; HumanEval 33; MBPP 86; HotpotQA 200; DROP 200 (Table~\ref{tab:datasets}). \\*
& Test samples & Full set
& Full test split for final reporting. \\*
& Max concurrency & 50
& Maximum number of concurrent evaluations (API/hardware dependent). \\*
& Evaluations per node & 1 / 3
& Root nodes are evaluated once; fused offspring are evaluated 3 times and averaged to reduce variance. \\
\end{longtable}
\endgroup

%==============================================================================
\section{Dataset Details}
\label{app:datasets}
%==============================================================================

This section provides additional dataset information beyond \cref{sec:experiments}, with a focus on the
\emph{online reward estimation} protocol used during search.
To ensure a fair, apples-to-apples comparison, \textbf{we strictly follow the AFlow evaluation protocol for
(i) dataset partitioning, (ii) answer extraction / code execution, and (iii) metric computation}.
Concretely, for each benchmark we use the same ``divided test set'' setup as AFlow: a \emph{fixed} held-out
validation subset is used exclusively to provide online reward feedback during optimization, while the
final numbers are reported on a disjoint test subset. \citep{zhang2025aflow}

\begin{table}[H]
\centering
\caption{Benchmark datasets and evaluation protocol used in our experiments (aligned with AFlow).
\textbf{Val Size} is the fixed validation subset used for reward estimation during search (held fixed across iterations),
and \textbf{Test Size} is the disjoint evaluation subset used for reporting final results.}
\label{tab:datasets}
\small
\setlength{\tabcolsep}{4pt}
\begin{tabular}{lccccl}
\toprule
\textbf{Dataset} & \textbf{Domain} & \textbf{Val Size} & \textbf{Test Size} & \textbf{Metric} & \textbf{Notes} \\
\midrule
GSM8K & Math Reasoning & 264 & 1,055 & Accuracy & Grade-school word problems \\
MATH & Math Reasoning & 119 & 486 & Accuracy & Level 5 (hardest) problems only \\
HumanEval & Code Generation & 33 & 131 & Pass@1 & Python function synthesis \\
MBPP & Code Generation & 86 & 341 & Pass@1 & Mostly Basic Python Problems \\
HotpotQA & Question Answering & 200 & 800 & F1 & Multi-hop reasoning required \\
DROP & Question Answering & 200 & 800 & F1 & Discrete reasoning over paragraphs \\
\bottomrule
\end{tabular}
\end{table}

\paragraph{Evaluation Metrics (standard metrics; computed with the AFlow harness).}
We adopt the \emph{canonical} metrics used by these benchmarks, and compute them using the same evaluation
pipeline as AFlow. \citep{zhang2025aflow}
For mathematical reasoning (GSM8K, MATH), we report exact-match accuracy (a.k.a.\ solve rate) after answer extraction.
GSM8K is evaluated by extracting a final numeric answer from the model output and matching it to the ground truth
(with a small numeric tolerance to handle formatting). \citep{cobbe2021training}
For MATH, the benchmark is designed to be scored by exact match after normalization of the final answer, and we
follow the standard practice of extracting the final boxed answer and checking equivalence (including symbolic
checking for algebraic expressions when needed). \citep{hendrycksmath2021}

For code generation (HumanEval, MBPP), we report Pass@1, the standard unit-test-based functional correctness metric:
a solution is counted as correct if the first generated program passes all provided tests in a sandboxed environment.
This is the established evaluation protocol for HumanEval and widely used for MBPP-style program synthesis benchmarks. \citep{chen2021codex,austin2021program}
We apply code sanitization to extract the target function and run tests with timeouts, matching the AFlow setup. \citep{zhang2025aflow}

For question answering (HotpotQA, DROP), we compute token-level F1 after standard text normalization (lowercasing,
article removal, punctuation stripping, and whitespace normalization). HotpotQA and DROP are commonly reported with
both EM and F1 in their official evaluations; following AFlow, we use F1 as the primary metric in our main tables. \citep{yang-etal-2018-hotpotqa,dua-etal-2019-drop,zhang2025aflow}

\paragraph{Sampling Strategy.}
Following AFlow, validation instances are randomly sampled \emph{once} and then held fixed across all iterations,
so that online rewards are comparable over time and across methods. \citep{zhang2025aflow}
For MATH, we restrict to Level 5 (the hardest subset) to obtain a challenging optimization target.
Every validation execution evaluates the entire fixed validation subset in \cref{tab:datasets}: 264, 119, 33, 86, 200, and 200 instances for GSM8K, MATH, HumanEval, MBPP, HotpotQA, and DROP, respectively. No additional per-round validation subsampling is used.
Final evaluation is performed on the full held-out test subset in \cref{tab:datasets}.
%==============================================================================
\section{LLM Prompt Templates}
\label{app:prompts}
%==============================================================================

This section summarizes the core LLM prompt templates used in DAGO.
We employ role-specialized prompts for: (i) \emph{strategy discovery}, (ii) \emph{initial population generation}, (iii) \emph{parent advantage summarization}, and (iv) \emph{multi-parent fusion}.
To make the optimization process robust, all generation prompts enforce a strict XML response schema that is machine-parseable into \texttt{workflow\_code} and \texttt{prompt} fields (Appendix~\ref{app:algorithm}).

\paragraph{Prompt formatting (reproducibility).}
For readability, we present prompts using a consistent boxed style with syntax highlighting.

\Needspace{8\baselineskip}
\subsection{Shared Constraints and Output Schema}
\label{app:prompt_constraints}

All optimizer-facing prompts append a shared constraint block that (a) defines the required workflow interface, (b) documents the operator return conventions, (c) specifies evaluator interface requirements (output format and parsing conventions), and (d) enforces the XML output schema.

\begin{tcolorbox}[title={\textbf{\small Shared Constraints and Interface Specification}}, boxrule=1pt, arc=0mm, colback=black!5!white, colframe=black!75!white, breakable, before skip=10pt, after skip=10pt, pad at break=2mm, parbox=false]
\begingroup
\fvset{fontsize=\scriptsize, breaklines, breakanywhere, frame=lines, framesep=2mm, tabsize=4}
\input{minted-cache/96B5477A74A2FDBC45B1729C20CE6D55.highlight.minted}
\endgroup
\end{tcolorbox}

\Needspace{8\baselineskip}
\subsection{Strategy Discovery Prompt}
\label{app:prompt_strategy}

The strategy discovery prompt generates diverse workflow strategies for a given task type:

\begin{tcolorbox}[title={\textbf{\small Strategy Discovery Template}}, boxrule=1pt, arc=0mm, colback=black!5!white, colframe=black!75!white, breakable, before skip=10pt, after skip=10pt, pad at break=2mm, parbox=false]
\begingroup
\fvset{fontsize=\scriptsize, breaklines, breakanywhere, frame=lines, framesep=2mm, tabsize=4}
\begin{MintedVerbatim}[commandchars=\\\{\}]
Generate **\PYGZob{}num\PYGZus{}strategies\PYGZcb{}** structurally distinct LLM workflow strategies
for the **\PYGZob{}task\PYGZus{}display\PYGZus{}name\PYGZcb{}** benchmark.

**Context**: You are designing LLM\PYGZhy{}based workflows to solve problems from
the \PYGZob{}task\PYGZus{}display\PYGZus{}name\PYGZcb{} dataset. Each strategy should describe how to
orchestrate LLM calls to solve these problems effectively.

**Requirements**:
1. **Topological Diversity**: Strategies must differ in **control flow**
   (One\PYGZhy{}Shot vs. Linear Chain vs. Recursive/Loop).
2. **Mechanism Fit**: Address specific failure modes of this benchmark
   (e.g., Code \PYGZhy{}\PYGZgt{} Test\PYGZhy{}Driven; Math \PYGZhy{}\PYGZgt{} Step\PYGZhy{}Verification).
3. **Simplicity First**: Default to the simplest effective flow.

**Output JSON**:
[\PYGZob{}
  \PYGZdq{}name\PYGZdq{}: \PYGZdq{}Strategy Name\PYGZdq{},
  \PYGZdq{}control\PYGZus{}flow\PYGZus{}type\PYGZdq{}: \PYGZdq{}One\PYGZhy{}Shot|Linear|Recursive|Tree\PYGZdq{},
  \PYGZdq{}core\PYGZus{}approach\PYGZdq{}: \PYGZdq{}Brief execution steps\PYGZdq{},
  \PYGZdq{}task\PYGZus{}adaptation\PYGZdq{}: \PYGZdq{}Specific failure mode addressed\PYGZdq{},
  \PYGZdq{}complexity\PYGZus{}profile\PYGZdq{}: \PYGZdq{}low|medium|high\PYGZdq{}
\PYGZcb{}]
\end{MintedVerbatim}

\endgroup
\end{tcolorbox}

\Needspace{8\baselineskip}
\subsection{Initial Population Generation Prompt}
\label{app:prompt_initial}

Given a sampled strategy, DAGO prompts the optimizer LLM to generate an executable workflow and its internal prompt set:

\begin{tcolorbox}[title={\textbf{\small Initial Workflow Generation Template}}, boxrule=1pt, arc=0mm, colback=black!5!white, colframe=black!75!white, breakable, before skip=10pt, after skip=10pt, pad at break=2mm, parbox=false]
\begingroup
\fvset{fontsize=\scriptsize, breaklines, breakanywhere, frame=lines, framesep=2mm, tabsize=4}
\begin{MintedVerbatim}[commandchars=\\\{\}]
Design a workflow for **\PYGZob{}task\PYGZus{}type\PYGZcb{}** using strategy: **\PYGZob{}strategy\PYGZus{}choice\PYGZcb{}**.

Workflow Design:
\PYGZhy{} Target 1\PYGZhy{}5 LLM calls (max 8 if truly needed)
\PYGZhy{} Each call must serve a distinct, necessary purpose
\PYGZhy{} Simpler workflows often perform better than complex ones

Inner Prompt Design (EQUALLY IMPORTANT):
\PYGZhy{} Inner prompts determine output quality \PYGZhy{} design them carefully
\PYGZhy{} The ENDING of each prompt is critical \PYGZhy{} it controls output format
\PYGZhy{} Follow the output format in EVALUATOR INTERFACE

\PYGZob{}shared\PYGZus{}constraints\PYGZcb{}  \PYGZsh{} includes operator docs + evaluator interface + XML schema
\end{MintedVerbatim}

\endgroup
\end{tcolorbox}

\Needspace{8\baselineskip}
\subsection{Parent Advantage Summarization Prompt}
\label{app:prompt_summarize}

Before fusion, DAGO analyzes each parent workflow to extract strengths and weaknesses:

\begin{tcolorbox}[title={\textbf{\small Parent Advantage Summarization Template}}, boxrule=1pt, arc=0mm, colback=black!5!white, colframe=black!75!white, breakable, before skip=10pt, after skip=10pt, pad at break=2mm, parbox=false]
\begingroup
\fvset{fontsize=\scriptsize, breaklines, breakanywhere, frame=lines, framesep=2mm, tabsize=4}
\input{minted-cache/EE4994BDDEC461DEC865505480228BC2.highlight.minted}
\endgroup
\end{tcolorbox}

\Needspace{8\baselineskip}
\subsection{Multi-Parent Fusion Prompt}
\label{app:prompt_fusion}

The fusion prompt synthesizes a new workflow by combining strengths from multiple parents:

\begin{tcolorbox}[title={\textbf{\small Multi-Parent Fusion Template}}, boxrule=1pt, arc=0mm, colback=black!5!white, colframe=black!75!white, breakable, before skip=10pt, after skip=10pt, pad at break=2mm, parbox=false]
\begingroup
\fvset{fontsize=\scriptsize, breaklines, breakanywhere, frame=lines, framesep=2mm, tabsize=4}
\begin{MintedVerbatim}[commandchars=\\\{\}]
Create a workflow for **\PYGZob{}task\PYGZus{}type\PYGZcb{}** by improving the BEST parent.

**Parent Workflows** (sorted by score, highest first):
\PYGZob{}parent\PYGZus{}workflows\PYGZus{}info\PYGZcb{}

=== MANDATORY FUSION PROTOCOL ===

**Step 1: COPY the highest\PYGZhy{}scoring parent as your BASELINE**
\PYGZhy{} Your default output should be IDENTICAL to the best parent
\PYGZhy{} Do not \PYGZdq{}rewrite\PYGZdq{} or \PYGZdq{}improve\PYGZdq{} unless you have concrete evidence

**Step 2: Review each parent\PYGZsq{}s FULL analysis summary**
\PYGZhy{} Understand structure (llm\PYGZus{}calls, flow), prompt\PYGZus{}analysis, and evaluator\PYGZus{}alignment
\PYGZhy{} Pay attention to `weaknesses` and `fusion\PYGZus{}guidance`
\PYGZhy{} Identify at most 2\PYGZhy{}3 SMALL, TARGETED changes that address documented issues

**Step 3: For EACH proposed change, answer:**
\PYGZhy{} [ ] What specific failure does this fix? (cite from parent analysis)
\PYGZhy{} [ ] Could this break something that currently works?
\PYGZhy{} [ ] Is this change truly necessary, or am I adding complexity for no reason?
\PYGZhy{} If ANY answer is uncertain \PYGZhy{}\PYGZgt{} DO NOT make this change

**Step 4: Inner Prompt Changes (CRITICAL)**
\PYGZhy{} If changing a prompt ending, quote BOTH the original and your new version
\PYGZhy{} Explain why the new wording better matches evaluator behavior
\PYGZhy{} The last 1\PYGZhy{}2 sentences of each prompt are the most important

=== CHANGE BUDGET ===
\PYGZhy{} **Modifications**: Maximum 2\PYGZhy{}3 targeted changes
\PYGZhy{} **LLM calls**: Keep similar to best parent (\PYGZlt{}=8 total); adding calls requires justification

Types of allowed changes:
1. Fix format instruction \PYGZhy{} make prompt ending match evaluator expectations
2. Remove unnecessary complexity \PYGZhy{} delete redundant LLM calls
3. Fix documented bug \PYGZhy{} address specific weakness from parent analysis

Types of FORBIDDEN changes:
\PYGZhy{} Adding new verification steps \PYGZdq{}just to be safe\PYGZdq{}
\PYGZhy{} Rewriting prompts that already work well
\PYGZhy{} Changing control flow without evidence of improvement
\PYGZhy{} Adding complexity because it \PYGZdq{}sounds better\PYGZdq{}

=== EVALUATOR INTERFACE FOR \PYGZob{}task\PYGZus{}type\PYGZcb{} ===
\PYGZob{}evaluator\PYGZus{}interface\PYGZcb{}  \PYGZsh{} output format requirements

\PYGZob{}shared\PYGZus{}constraints\PYGZcb{}

**OUTPUT**: Think through your changes first, then provide the XML output.
\end{MintedVerbatim}

\endgroup
\end{tcolorbox}

%==============================================================================

\section{Implementation Details of Bandit-Guided Fusion}
\label{app:algorithm}

This appendix specifies the selection loop corresponding to Eqs.~\eqref{eq:proposal_distribution}--\eqref{eq:diagonal_update}.
It distinguishes workflow evaluations from arm scoring, preserves the diagonal numerical update, and makes the budget and feature dimensions explicit.

\subsection{State, dimensions, and reported configuration}
\label{app:implementation_state}

Each node stores an immutable identifier, workflow code $\mathcal{C}_i$, internal prompts $\mathcal{P}_i$, an observed validation score $s_i$, the evaluation repetition count, its fixed embedding $\mathbf{e}_i$, and parent identifiers.
The active node pool includes the initial workflows and subsequently evaluated children.
The bandit state consists of vectors $\mathbf{d},\mathbf{b}\in\mathbb{R}^{p}$ with $p=Jd$; node embeddings and graph edges are stored separately.

\begin{table}[htbp]
\centering
\small
\caption{Reported settings used to specify the revised procedure. The table does not introduce new tuned parameters.}
\label{tab:method_settings}
\begin{tabular}{@{}lll@{}}
\toprule
Component & Parameter & Reported value \\
\midrule
Initialization & $N$, $q_{\mathrm{init}}$ & $10$, $1$ \\
Fusion budget & $T$, $q$, $B_{\max}$ & $30$, $3$, $100$ \\
Stopping & $T_{\min}$, $C_{\mathrm{patience}}$, $\delta_{\mathrm{stop}}$ & $20$, $5$, $10^{-4}$ \\
Candidate cap / arm size & $M$, $J$ & $200$, $5$ \\
Sampling & $\tau_{\mathrm{init}}\to\tau_{\mathrm{final}}$, $p_{\min}$ & $2.0\to0.5$, $0.01$ \\
Bandit & $\lambda$, $\nu$ & $0.2$, $0.1$ \\
Features & $d$, $p=Jd$ & $2560$, $12800$ \\
Representation & Encoder family & Qwen3-Embedding \\
Generation & Optimization model & Claude 3.5 Sonnet \\
Execution & Evaluation model & GPT-4o-mini-0718 \\
\bottomrule
\end{tabular}
\end{table}

The default configuration satisfies $N\ge J$ and $\binom{N}{J}=\binom{10}{5}=252$, so $M_t=200$ throughout the default search. More generally, the candidate count is capped at the number of available parent combinations, while $J$ and the feature dimension remain fixed within each search.
For a deployment with fewer than $J$ initial parents, initialization must be expanded or a fixed padding-and-mask convention must be specified before search; silently changing $J$ while retaining a $Jd$-dimensional state is invalid.
Such an extension is not needed for the reported setting and is not added here.

\subsection{Candidate-arm proposal}
\label{app:sampling}

\paragraph{Stable score sampling.}
For the $n_t=|\mathcal{V}_{t-1}|$ available nodes, compute
\begin{equation}
\begin{aligned}
 w_t(v)&=\exp\!\left(\frac{s(v)-s_{\max}}{\tau_t}\right),\quad
 p_t(v)=\frac{w_t(v)}{\sum_u w_t(u)},\\
 \widetilde p_t(v)&=\frac{\max\{p_t(v),p_{\min}\}}{Z_t},
 \quad Z_t=\sum_u\max\{p_t(u),p_{\min}\}.
\end{aligned}
\end{equation}
Subtracting $s_{\max}$ prevents numerical overflow without changing the distribution.
The renormalized probabilities obey $\widetilde p_t(v)\ge p_{\min}/Z_t>0$, not necessarily $\widetilde p_t(v)\ge p_{\min}$.
Since $1\le Z_t\le1+n_tp_{\min}$, the floor is a support safeguard, not a fixed lower bound on every arm's inclusion probability.
The temperature is computed from the planned maximum $T$, even when search stops early.

\paragraph{Without replacement inside an arm.}
To sample one arm, start with all nodes eligible, draw a parent with probability proportional to its remaining weight $\widetilde p_t(v)$, remove it, and repeat until $J$ distinct parents have been drawn.
Sort the result by decreasing validation score and then increasing node identifier.
For an unordered parent set $S$ of size $J$, the probability of its canonical arm is
\begin{equation}
 Q_t(S)=\sum_{\pi\in\mathrm{Perm}(S)}
 \prod_{j=1}^{J}
 \frac{\widetilde p_t(v_{\pi_j})}
 {1-\sum_{\ell<j}\widetilde p_t(v_{\pi_\ell})}.
 \label{eq:arm_proposal_probability}
\end{equation}
Thus parent draws within an arm are not independent.
The canonical ordering does not affect membership but fixes the feature positions used by the selector.

\paragraph{Round-dependent number of distinct candidates.}
Let $F_t=\binom{n_t}{J}$ be the number of feasible parent combinations before round $t$.
The number of distinct candidates is
\begin{equation}
 M_t=\min\{M,F_t\}
 =\min\left\{M,\binom{|\mathcal{V}_{t-1}|}{J}\right\}.
 \label{eq:candidate_count}
\end{equation}
The candidate set starts empty at each round.
Repeatedly sample one arm from the fixed distribution $Q_t$ in Eq.~\eqref{eq:arm_proposal_probability}, canonicalize it, and insert it only if its parent-identifier sequence has not already been accepted in this round.
Continue until $|\mathcal{A}_t|=M_t$.
Thus $M$ is a \emph{cap}, whereas $M_t$ is the actual number of unique candidate arms, not the number of raw sampling attempts.
Duplicates are determined by parent identities, not by whether different nodes contain identical workflow text.
The same arm may be proposed again in a later round.
Raw proposals are independent conditional on the archive and sampling weights, but accepted arms are dependent with the law in Eq.~\eqref{eq:unique_arm_conditional}; weights are not recomputed after rejected proposals.

When $F_t<M$, all feasible parent combinations are included, without changing $J$ or allowing repeated parents.
For the arm-size ablations with $N=10$ and $M=200$, the first-round candidate counts for $J\in\{1,3,5,10\}$ are $10$, $120$, $200$, and $1$, respectively.
Positive arm probabilities ensure almost-sure termination, although there is no deterministic upper bound on the number of raw proposals.

If $L_t$ is the total number of raw proposals and $\mathcal{U}_m$ the first $m$ accepted arms, the next acceptance has geometric waiting time with success probability $1-Q_t(\mathcal{U}_m)$.
Consequently,
\begin{equation}
 \mathbb{E}[L_t]
 =\sum_{m=0}^{M_t-1}\mathbb{E}\!\left[\frac{1}{1-Q_t(\mathcal{U}_m)}\right]
 \le \frac{1}{q_{\min,t}}\sum_{m=0}^{M_t-1}\frac{1}{F_t-m},
 \quad q_{\min,t}=\min_{a\in\bar{\mathcal{A}}_t}Q_t(a)>0.
 \label{eq:unique_proposal_attempts}
\end{equation}
This finite expectation may be large when proposal mass is concentrated.
Appendix~\ref{app:proposal_gap} analyzes coverage for this rejection-sampling procedure rather than treating the accepted candidates as independent.

\begin{algorithm}[htbp]
\caption{Candidate proposal with a round-dependent cap}
\label{alg:arm_proposal}
\small
\begin{algorithmic}[1]
\REQUIRE Nodes $\mathcal{V}_{t-1}$ with $n_t\ge J$; candidate cap $M\ge1$; $J,\tau_t,p_{\min}$
\STATE $M_t\leftarrow\min\{M,\binom{n_t}{J}\}$
\STATE Compute stable normalized weights $\widetilde p_t$; $\mathcal{A}_t\leftarrow\varnothing$
\WHILE{$|\mathcal{A}_t|<M_t$}
    \STATE Draw $J$ distinct nodes using sequential weighted sampling without replacement
    \STATE Canonically order the nodes by score and identifier to form arm $a$
    \IF{$a\notin\mathcal{A}_t$}
        \STATE $\mathcal{A}_t\leftarrow\mathcal{A}_t\cup\{a\}$
    \ENDIF
\ENDWHILE
\STATE \textbf{return} $\mathcal{A}_t$
\end{algorithmic}
\end{algorithm}

\subsection{Embeddings and the diagonal reward surrogate}
\label{app:embedding}
\label{app:linucb}

\paragraph{Fixed features and cached encodings.}
Construct $\mathbf{e}(v)=\mathrm{Enc}(\mathcal{C}(v)\oplus\mathcal{P}(v))$ once for each immutable workflow text and cache it.
The encoder is not fine-tuned by the bandit.
For arm $a=(v_{(1)},\ldots,v_{(J)})$, concatenate the ordered embeddings into $\mathbf{x}(a)\in\mathbb{R}^{Jd}$.
No interaction features, extra intercept, feature normalization, score clipping, or encoder retraining is introduced by the mathematical revision.
The encoder checkpoint, serialization of the prompt set, truncation length, and any preprocessing used by an implementation must be fixed across training and arm scoring and reported with the code.

\paragraph{Exact diagonal computation.}
Before round $t$, the sufficient statistics are
\begin{equation}
 d_{t-1,k}=\lambda+\sum_{s<t}x_{s,k}^2,
 \qquad b_{t-1,k}=\sum_{s<t}r_sx_{s,k},
 \qquad \widehat\theta_{t-1,k}=b_{t-1,k}/d_{t-1,k}.
 \label{eq:diagonal_statistics}
\end{equation}
Arm predictions and exploration proxies are
\begin{equation}
 m_t(a)=\sum_{k=1}^{p}x_k(a)\widehat\theta_{t-1,k},
 \qquad h_t(a)=\sqrt{\sum_{k=1}^{p}x_k(a)^2/d_{t-1,k}}.
\end{equation}
The selected arm maximizes $m_t(a)+\nu h_t(a)$.
After its child has been evaluated, add $\mathbf{x}_t\odot\mathbf{x}_t$ to $\mathbf{d}$ and $r_t\mathbf{x}_t$ to $\mathbf{b}$.
Each fusion decision contributes one observation even when its child's reward averages multiple validation executions.
Initial nodes do not create synthetic fusion observations.

\paragraph{Relationship to least squares.}
The coefficient vector minimizes the separable quadratic
\begin{equation}
 \frac12\boldsymbol{\theta}^{\top}\mathbf{D}_{t-1}\boldsymbol{\theta}
 -\mathbf{b}_{t-1}^{\top}\boldsymbol{\theta},
\end{equation}
which replaces the full design Gram matrix by its diagonal.
In contrast, multivariate ridge regression minimizes
$\frac12\sum_{s<t}(r_s-\mathbf{x}_s^\top\boldsymbol{\theta})^2+
\frac\lambda2\|\boldsymbol{\theta}\|_2^2$ and retains all cross-feature terms.
The distinction matters for both coefficients and confidence interpretation.

\paragraph{Cold start, histories, and nonstationarity.}
Before the first observed fusion reward, $\mathbf{b}_0=0$, so selection is driven by $\nu\|\mathbf{x}(a)\|_2/\sqrt\lambda$, with identity-based tie breaking.
Without a fixed feature-norm convention, the initial preference may reflect embedding norms.
Storing past arms and outcomes is useful for auditing, but recomputing Eq.~\eqref{eq:diagonal_statistics} from the same history yields the same state.
It is not a mechanism for adapting to reward drift.
Forgetting factors, sliding windows, or full-matrix refits would change the algorithm and require separate evaluation.

\subsection{LLM interaction and evaluation interface}
\label{app:llm_protocol}

\paragraph{Initialization.}
A strategy-discovery call supplies task-level design ideas to initial workflow generation.
Each generated root is evaluated with the fixed validation harness and inserted with its score and embedding.
The test set is excluded from these calls.

\paragraph{Summarization.}
For every selected parent, provide its code, internal prompts, and score to the summarization model.
The structured analysis records call structure, prompt endings, evaluator-interface risks, strengths, weaknesses, and restrictions on safe edits.
These are LLM-generated analyses rather than verified causal attributions.
They are inputs to fusion, not labels for the bandit and not substitutes for child validation.

\paragraph{Constrained multi-parent fusion.}
The fusion input contains all selected parent code and prompts together with their scores and summaries.
The existing template instructs the model to use the highest-scoring parent as the baseline, consider at most two or three targeted changes, and justify any additional calls or control-flow modifications.
Accordingly, ``fusion'' denotes generation conditioned on multiple parents, often through a local edit of the strongest parent; it does not require literal inheritance from every supplied parent.
Parent edges record conditioning provenance.
The prompt and output-schema constraints are generation instructions, not guarantees that the emitted program is valid or improves performance.

\paragraph{Bounded validation reward.}
With metric $\ell\in[0,1]$ and fixed validation set $\mathcal{D}_{\mathrm{val}}$,
\begin{equation}
 \mathrm{Eval}_{q}(w)=\frac1q\sum_{j=1}^{q}
 \frac1{|\mathcal{D}_{\mathrm{val}}|}
 \sum_{(x,y)\in\mathcal{D}_{\mathrm{val}}}
 \ell\bigl(w(x;\xi_j),y\bigr).
 \label{eq:repeated_validation}
\end{equation}
Each $j$ is an execution of the same generated program over the same split, with any execution randomness included in $\xi_j$.
The task-specific harness defines answer parsing, failed-instance scores, and timeouts.
The search score uses the mean over all declared executions, not the maximum or a selectively filtered subset.
Failure handling and retries must be fixed consistently across methods: a failed attempt must not be silently dropped to obtain uncharged additional generation or evaluation opportunities.
An entire-child failure can be assigned a bounded failure reward by the declared harness; this convention must be documented rather than inferred from the selection equations.

\subsection{Budget accounting and early stopping}
\label{app:budget_stopping}

\paragraph{What is counted.}
One budget unit is one execution of the complete validation split for a single workflow.
With $q_{\mathrm{init}}=1$ and $q=3$, a run with $K$ completed fusion rounds uses $B(K)=N+3K\le100$ declared validation-set executions, absent additional charged attempts.
Scoring $M_t$ candidate arms uses cached features and is not $M_t$ workflow evaluations.
Summarization and fusion still consume LLM resources, so equal validation budgets do not imply equal dollar cost.
Any charged retries must be included in the counter before another evaluation is permitted.

\paragraph{Recent-score stopping and incumbent retention.}
The stopping decision and the returned workflow use different information.
The incumbent is the workflow with the highest stored validation score across all initial and fusion nodes; it is updated whenever a new score is strictly higher, retaining the existing incumbent on ties.
Early stopping instead checks whether the most recent $C_{\mathrm{patience}}=5$ fusion-round validation scores show no relative improvement under the stopping tolerance $\delta_{\mathrm{stop}}=10^{-4}$.
Let
\[
 \mathbf{r}^{\mathrm{recent}}_t
 =(r_{t-C_{\mathrm{patience}}+1},\ldots,r_t)
\]
denote this recent-score window. Write $r_0=r_{t-C_{\mathrm{patience}}+1}$
and $r_+=\max\{r_{t-C_{\mathrm{patience}}+2},\ldots,r_t\}$.
Define
\[
 \mathrm{Stagnant}(\mathbf{r}^{\mathrm{recent}}_t;\delta_{\mathrm{stop}})
 =\begin{cases}
   (r_+-r_0)/r_0\le\delta_{\mathrm{stop}}, & r_0>0,\\
   r_+=0, & r_0=0.
 \end{cases}
\]
Thus the best score in the final four rounds is compared with the first score
of the five-round window. A positive score after a zero baseline counts as improvement.
The stopping gate is
\begin{equation}
 t\ge T_{\min}=20
 \quad\text{and}\quad
 \mathrm{Stagnant}(\mathbf{r}^{\mathrm{recent}}_t;\delta_{\mathrm{stop}})=\mathrm{true}.
 \label{eq:early_stopping_gate}
\end{equation}
This is not a counter of consecutive failures to exceed the best initial or historical score.
The window contains completed fusion-round rewards, each averaged over the declared validation executions; initialization evaluations are not window entries.
The minimum applies to completed fusion rounds, not to individual validation executions.
The reported budget permits all $20$ rounds: $Nq_{\mathrm{init}}+20q=70\le100$.
The run ends when the recent-score stopping gate is met or the iteration/evaluation budget is exhausted.

\begin{algorithm}[htbp]
\caption{End-to-end DAGO with explicit budget and stopping state}
\label{alg:dago_full}
\small
\begin{algorithmic}[1]
\REQUIRE $N\ge J$, $M\ge1$; $T,B_{\max}$; $q_{\mathrm{init}},q$; $\lambda,\nu$; sampling settings; $T_{\min},C_{\mathrm{patience}},\delta_{\mathrm{stop}}$
\STATE Check $Nq_{\mathrm{init}}\le B_{\max}$; initialize empty DAG and embedding cache
\STATE Discover task strategies; generate $N$ roots; evaluate each $q_{\mathrm{init}}$ times and insert
\STATE $B\leftarrow Nq_{\mathrm{init}}$; $p\leftarrow Jd$; $\mathbf{d}\leftarrow\lambda\mathbf{1}_p$; $\mathbf{b}\leftarrow\mathbf{0}_p$
\STATE $v^*\leftarrow\arg\max_{v\in\mathcal{V}_0}s(v)$; initialize an empty recent-reward window
\FOR{$t=1$ to $T$}
    \IF{$B+q>B_{\max}$}
        \STATE \textbf{break}
    \ENDIF
    \STATE Compute $\tau_t$ and $\widetilde p_t$; propose candidates using Algorithm~\ref{alg:arm_proposal}
    \STATE $\widehat{\boldsymbol{\theta}}\leftarrow\mathbf{b}\oslash\mathbf{d}$
    \STATE Form $\mathbf{x}(a)$ for candidates and select $a_t\in\arg\max_a[m_t(a)+\nu h_t(a)]$
    \STATE Set $\mathbf{x}_t\leftarrow\mathbf{x}(a_t)$; summarize all selected parents
    \STATE Generate child $w_t$ using parent code, prompts, scores, and summaries
    \STATE $r_t\leftarrow\mathrm{Eval}_{q}(w_t)$; $B\leftarrow B+q$; record the decision and outcome
    \STATE $\mathbf{d}\leftarrow\mathbf{d}+\mathbf{x}_t\odot\mathbf{x}_t$; $\mathbf{b}\leftarrow\mathbf{b}+r_t\mathbf{x}_t$
    \STATE Cache the child's embedding; insert it and all parent edges into the DAG
    \IF{$r_t>s(v^*)$}
        \STATE $v^*\leftarrow v_{\mathrm{new}}$
    \ENDIF
    \STATE Append $r_t$ to the recent-reward window; retain only its last $C_{\mathrm{patience}}$ entries
    \IF{$t\ge T_{\min}$ and $\mathrm{Stagnant}(\mathbf{r}^{\mathrm{recent}}_t;\delta_{\mathrm{stop}})$}
        \STATE \textbf{break}
    \ENDIF
\ENDFOR
\STATE \textbf{return} $v^*$ and the archive
\end{algorithmic}
\end{algorithm}

\subsection{Computational costs and full-matrix comparison}
\label{app:computational_details}

\paragraph{Diagonal selection cost.}
The bandit stores $O(p)$ statistics, scores one arm in $O(p)$ time, scores $M_t$ arms in $O(M_t p)$ time, and updates in $O(p)$ time.
Candidate features can be constructed and scored in a stream, avoiding an additional $O(M_t p)$ materialized feature matrix.
The archive additionally uses $O(nd)$ embedding storage and $O(KJ)$ parent links, apart from workflow text and logs.
Computing parent proposal weights costs $O(n)$. If one raw weighted draw costs $C_{\mathrm{draw}}(n,J)$, proposal costs $O(n+L_t[C_{\mathrm{draw}}(n,J)+J\log J])$ for the realized $L_t$ attempts, using hashing for duplicate rejection. Storing the accepted parent identifiers uses $O(M_t J)$ space. Only the $M_t$ accepted arms need bandit scoring; $L_t$ need not equal $M_t$.
LLM calls, execution, and initial embedding computation are separate from the bandit arithmetic.

\paragraph{Dense full covariance is not the only exact implementation.}
At $p=12{,}800$, a dense full Gram matrix has $p^2=163{,}840{,}000$ entries.
However, with $K\le30$ observations, let $\mathbf{X}\in\mathbb{R}^{K\times p}$ contain their features.
The Woodbury identity gives
\begin{equation}
 (\lambda\mathbf{I}+\mathbf{X}^{\top}\mathbf{X})^{-1}
 =\lambda^{-1}\mathbf{I}
 -\lambda^{-1}\mathbf{X}^{\top}
   (\lambda\mathbf{I}_{K}+\mathbf{X}\mathbf{X}^{\top})^{-1}\mathbf{X}.
 \label{eq:woodbury_ridge}
\end{equation}
Thus exact full-ridge predictions and uncertainty can be obtained using an $O(Kp+K^2)$ representation instead of storing $p^2$ entries.
For example,
\begin{equation}
 h_{\mathrm{full}}(\mathbf{x})^2
 =\lambda^{-1}\left[\|\mathbf{x}\|_2^2-
   (\mathbf{X}\mathbf{x})^\top
   (\lambda\mathbf{I}_{K}+\mathbf{X}\mathbf{X}^{\top})^{-1}
   (\mathbf{X}\mathbf{x})\right].
\end{equation}
After the small-system factorization, evaluating this expression for a new arm costs $O(Kp+K^2)$ using straightforward matrix-vector products and triangular solves, compared with $O(p)$ for the diagonal proxy.
The diagonal choice is therefore a speed-and-modeling trade-off, not evidence that exact full-ridge computation is impossible.
An observed difference between diagonal and full-matrix ablations does not by itself establish a statistical explanation for that difference.

\paragraph{Reproducibility boundaries.}
Candidate de-duplication uses the round-dependent count $M_t$, and early stopping applies to the most recent five round scores after the $t\ge20$ gate. Exact reproduction also requires model revisions, decoding settings, validation splits, encoder preprocessing, and failure/retry accounting.
These settings should accompany the released implementation and run logs.
Additional nonlinear surrogates, interaction features, forgetting mechanisms, or diagnostic probes are not silently included in the reported configuration.

\section{Comprehensive Analysis of Discovered Workflows}
\label{app:workflows}
%==============================================================================

This section presents a detailed case study of the optimal workflows discovered by DAGO across all six benchmark datasets, with comparative analysis against AFlow where applicable. We analyze workflow structures, multi-parent fusion lineages, emergent design patterns, and provide insights into why DAGO's DAG-based optimization discovers more effective solutions.
Unless otherwise specified, scores in this appendix are \emph{validation scores used during optimization}, not the final test results in \cref{tab:main}. Each initial workflow is evaluated once on the fixed validation split ($q_{\mathrm{init}}=1$); each fusion workflow is evaluated three times on that same split, and its mean score is stored as the reward ($q=3$). Separately, the final selected workflow is executed three times on the held-out test split for reporting.

\subsection{Cross-Dataset Optimization Statistics}
\label{app:workflow_stats}

\cref{tab:workflow_stats_detail} summarizes the optimization trajectories across all benchmarks, revealing several important patterns about DAGO's search dynamics.

\begin{table}[h]
\centering
\caption{DAGO optimization statistics per benchmark. Initial scores use one execution of the fixed validation split; fusion scores average three executions. Relative improvement is measured against the best initial score. Fusion-reward standard deviations describe variation across evaluated fusion nodes, not variation across independent searches or test executions.}
\label{tab:workflow_stats_detail}
\setlength{\tabcolsep}{4pt}
\begin{tabular}{lcccccc}
\toprule
\textbf{Metric} & \textbf{GSM8K} & \textbf{MATH} & \textbf{HumanEval} & \textbf{MBPP} & \textbf{HotpotQA} & \textbf{DROP} \\
\midrule
Total Iterations & 28 & 27 & 25 & 25 & 30 & 30 \\
Early Stopped & Yes & Yes & Yes & Yes & No & No \\
Total DAG Nodes & 38 & 37 & 35 & 35 & 40 & 40 \\
Total DAG Edges & 140 & 135 & 125 & 125 & 150 & 150 \\
\midrule
Best Initial Score & 95.8\% & 50.4\% & 87.9\% & 83.7\% & 77.2\% & 82.9\% \\
Best Fusion Score & 95.8\% & 55.7\% & 88.9\% & 86.0\% & 79.3\% & 84.0\% \\
Relative Improvement & +0.0\% & +10.5\% & +1.1\% & +2.7\% & +2.7\% & +1.3\% \\
\midrule
Best Fusion Node ID & 11 & 15 & 12 & 23 & 29 & 25 \\
Iteration Found & 11 & 15 & 12 & 23 & 29 & 25 \\
Parents of Best & 5 & 5 & 5 & 5 & 5 & 5 \\
\midrule
Mean Fusion Reward & 93.9\% & 49.6\% & 86.0\% & 78.4\% & 70.8\% & 77.5\% \\
Fusion Reward Std Dev & 1.0\% & 2.7\% & 2.3\% & 16.3\% & 12.0\% & 6.0\% \\
\bottomrule
\end{tabular}
\end{table}

\paragraph{Key Observations from Optimization Dynamics.}
Several important patterns emerge from \cref{tab:workflow_stats_detail}:

\textbf{(1) Task-Dependent Stopping Behavior.} GSM8K, MATH, HumanEval, and MBPP met the recent-score stopping criterion after 25--28 fusion rounds, whereas HotpotQA and DROP reached the 30-round limit. These are stopping outcomes under the reported search protocol, rather than evidence that a local optimum has been reached. The stopping criterion examines recent round scores separately from the archive-wide best score.

\textbf{(2) Consistent Multi-Parent Fusion.} All best fusion workflows are produced by 5-parent fusion operations ($J{=}5$), demonstrating that DAGO's multi-parent mechanism is essential for combining complementary strengths from multiple workflows. This contrasts sharply with AFlow's single-parent mutation, which cannot systematically recombine modules from multiple lineages.

\textbf{(3) Fusion Timing Patterns.} Best workflows were discovered at diverse iteration points (11--29), indicating that premature convergence is not a concern. The bandit-guided exploration successfully balances exploitation of promising regions with exploration of novel combinations.

%==============================================================================

\subsection{Mathematical Reasoning: MATH Dataset}
\label{app:workflow_math}
%==============================================================================

\paragraph{Best MATH Workflow (Node 15, Validation Score: 55.7\%).}
The optimal MATH workflow emerged from 5-parent fusion at iteration 15, combining nodes $\{1, 9, 2, 3, 4\}$ with scores $\{44.5\%, 42.9\%, 46.2\%, 50.4\%, 38.7\%\}$. Notably, this fusion surpassed all parents by a substantial margin (+5.3\% over the best parent), demonstrating DAGO's ability to synthesize complementary problem-solving strategies.

\begin{tcolorbox}[title={\textbf{\small DAGO-discovered MATH workflow (node\_15)}}, boxrule=1pt, arc=0mm, colback=black!5!white, colframe=black!75!white, breakable, before skip=10pt, after skip=10pt, pad at break=2mm, parbox=false]
\begingroup
\fvset{fontsize=\scriptsize, breaklines, breakanywhere, frame=lines, framesep=2mm, tabsize=4}
\input{minted-cache/5C73D1C522FD67EFF83B77E5FF2067BC.highlight.minted}
\endgroup
\end{tcolorbox}

\Needspace{8\baselineskip}
\paragraph{Associated Prompts.}
The workflow employs two carefully designed prompts that emerged through multi-parent fusion:

\begin{tcolorbox}[title={\textbf{\small MATH workflow prompts (node\_15)}}, boxrule=1pt, arc=0mm, colback=black!5!white, colframe=black!75!white, breakable, before skip=10pt, after skip=10pt, pad at break=2mm, parbox=false]
\begingroup
\fvset{fontsize=\scriptsize, breaklines, breakanywhere, frame=lines, framesep=2mm, tabsize=4}
\begin{MintedVerbatim}[commandchars=\\\{\}]
SOLVE\PYGZus{}PROMPT = \PYGZdq{}\PYGZdq{}\PYGZdq{}Solve this math problem step by step:
1. Break down the problem into clear steps
2. Show all work and calculations
3. Simplify expressions fully \PYGZhy{} ALWAYS use fractions (\PYGZbs{}frac\PYGZob{}a\PYGZcb{}\PYGZob{}b\PYGZcb{})
   instead of decimals
4. Use proper mathematical notation (\PYGZbs{}sqrt\PYGZob{}\PYGZcb{}, \PYGZbs{}pi, etc.)
5. Keep integers as integers
6. Put final answer in \PYGZbs{}boxed\PYGZob{}\PYGZcb{}, simplified. Use ONE \PYGZbs{}boxed\PYGZob{}\PYGZcb{}.\PYGZdq{}\PYGZdq{}\PYGZdq{}

ALT\PYGZus{}SOLVE\PYGZus{}PROMPT = \PYGZdq{}\PYGZdq{}\PYGZdq{}Solve this math problem using a different
approach if possible:
1. Try an alternative solution method
2. Show complete work
3. Verify calculations
4. Simplify completely \PYGZhy{} ALWAYS use fractions (\PYGZbs{}frac\PYGZob{}a\PYGZcb{}\PYGZob{}b\PYGZcb{})
   instead of decimals
5. Keep integers as integers
6. Put final answer in \PYGZbs{}boxed\PYGZob{}\PYGZcb{}, simplified. Use ONE \PYGZbs{}boxed\PYGZob{}\PYGZcb{}.\PYGZdq{}\PYGZdq{}\PYGZdq{}
\end{MintedVerbatim}

\endgroup
\end{tcolorbox}

\textbf{Prompt Design Insights:} The prompts explicitly enforce fraction notation over decimals, addressing a common failure mode where LLMs output decimal approximations that fail exact-match evaluation. The dual-prompt strategy (primary solve + alternative approach) encourages methodological diversity while maintaining consistent output formatting.

\paragraph{Multi-Parent Fusion Analysis.}
Examining the parent workflows reveals how fusion synthesized complementary strategies:
\begin{itemize}[nosep,leftmargin=*]
    \item \textbf{Node 1} (44.5\%): Decomposition-based approach with problem breakdown and computational verification.
    \item \textbf{Node 9} (42.9\%): Expert-style solving with programmatic verification and format post-processing.
    \item \textbf{Node 2} (46.2\%): Multi-approach generation with iterative solution exploration.
    \item \textbf{Node 3} (50.4\%): Solve-verify-ensemble with alternative solution generation (highest-scoring parent).
    \item \textbf{Node 4} (38.7\%): Direct computational approach via Programmer operator.
\end{itemize}

The optimal workflow inherited the ensemble mechanism from Node 3, the computational verification from Nodes 1 and 9, and the multi-solution diversity strategy from Node 2. This combination enables robust error detection through cross-validation while maintaining solution diversity.

\Needspace{8\baselineskip}
\paragraph{Comparison with AFlow MATH Workflow.}
AFlow's published MATH workflow employs a different strategy:

\begin{tcolorbox}[title={\textbf{\small AFlow MATH workflow (from published results)}}, boxrule=1pt, arc=0mm, colback=black!5!white, colframe=black!75!white, breakable, before skip=10pt, after skip=10pt, pad at break=2mm, parbox=false]
\begingroup
\fvset{fontsize=\scriptsize, breaklines, breakanywhere, frame=lines, framesep=2mm, tabsize=4}
\input{minted-cache/5C8CC85B0B23BB57005423356BCE1734.highlight.minted}
\endgroup
\end{tcolorbox}

\textbf{Key Structural Differences:}
\begin{enumerate}[nosep,leftmargin=*]
    \item \textbf{Verification Integration}: DAGO verifies the initial solution in the context of the original problem and passes the resulting feedback to the alternative-solution stage.
    \item \textbf{Solution Coupling}: Each alternative receives the original problem, the initial solution, and the verification feedback. The final ensemble combines the initial solution with these feedback-conditioned alternatives, whereas the displayed AFlow workflow generates its candidates without these inter-candidate dependencies.
    \item \textbf{Operator Invocations}: The displayed DAGO workflow uses five high-level operator invocations (one solve, one verification, two alternatives, and one ensemble), compared with six in the displayed AFlow workflow. These counts refer to workflow-level operators, not the number of underlying model or tool calls within each operator.
\end{enumerate}

%==============================================================================

\subsection{Mathematical Reasoning: GSM8K Dataset}
\label{app:workflow_gsm8k}
%==============================================================================

\paragraph{Best GSM8K Fusion Workflow (Node 11, Validation Score: 95.8\%).}
GSM8K represents a case where the initial population already achieved near-optimal performance (95.8\%), and fusion successfully maintained this performance through intelligent recombination. Node 11 emerged from parents $\{4, 1, 0, 7, 6\}$ at iteration 11.

\begin{tcolorbox}[title={\textbf{\small DAGO-discovered GSM8K workflow (node\_11)}}, boxrule=1pt, arc=0mm, colback=black!5!white, colframe=black!75!white, breakable, before skip=10pt, after skip=10pt, pad at break=2mm, parbox=false]
\begingroup
\fvset{fontsize=\scriptsize, breaklines, breakanywhere, frame=lines, framesep=2mm, tabsize=4}
\input{minted-cache/5D97962C2041B662469803462E4DDF55.highlight.minted}
\endgroup
\end{tcolorbox}

\Needspace{8\baselineskip}
\paragraph{Associated Prompts.}
The workflow employs two complementary prompts targeting different reasoning styles:

\begin{tcolorbox}[title={\textbf{\small GSM8K workflow prompts (node\_11)}}, boxrule=1pt, arc=0mm, colback=black!5!white, colframe=black!75!white, breakable, before skip=10pt, after skip=10pt, pad at break=2mm, parbox=false]
\begingroup
\fvset{fontsize=\scriptsize, breaklines, breakanywhere, frame=lines, framesep=2mm, tabsize=4}
\begin{MintedVerbatim}[commandchars=\\\{\}]
DIRECT\PYGZus{}SOLVE = \PYGZdq{}\PYGZdq{}\PYGZdq{}Solve this math word problem using clear logical
reasoning:
1. Read the problem carefully
2. Break down the key information
3. Solve step by step using plain language
4. Double check your calculations
5. State ONLY the final numerical answer. No units, no words,
   no punctuation. Just the number.\PYGZdq{}\PYGZdq{}\PYGZdq{}

MATH\PYGZus{}SOLVE = \PYGZdq{}\PYGZdq{}\PYGZdq{}Solve this math word problem using mathematical
equations:
1. Identify the variables and unknowns
2. Write and solve equations
3. Show your work clearly
4. Verify your solution
5. State ONLY the final numerical answer. No units, no words,
   no punctuation. Just the number.\PYGZdq{}\PYGZdq{}\PYGZdq{}
\end{MintedVerbatim}

\endgroup
\end{tcolorbox}

\textbf{Prompt Design Insights:} The dual-prompt strategy exploits complementary reasoning modes: DIRECT\_SOLVE encourages intuitive step-by-step reasoning, while MATH\_SOLVE promotes formal algebraic manipulation. Combined with the Programmer operator's computational verification, this creates a three-way ensemble that captures diverse solution paths. The strict output format (``Just the number'') addresses GSM8K's exact-match evaluation requirements.

\paragraph{Convergence Analysis.}
The tied scores between nodes 11 and 19 (both 95.8\%) reveal an important pattern: multiple structurally different workflows can achieve equivalent performance on well-structured mathematical reasoning tasks. This suggests that GSM8K problems have clear solution paths that multiple reasonable strategies can exploit. The key insight is that DAGO's exploration discovered this equivalence class while maintaining diversity, whereas single-parent methods might converge to a single solution path.

%==============================================================================

\subsection{Code Generation: HumanEval Dataset}
\label{app:workflow_humaneval}
%==============================================================================

\paragraph{Best HumanEval Fusion Workflow (Node 12, Validation Score: 88.9\%).}
The optimal HumanEval workflow emerged from 5-parent fusion at iteration 12, combining nodes $\{9, 0, 7, 8, 11\}$. This workflow introduces a sophisticated \emph{conditional retry} mechanism that distinguishes it from simpler approaches.

\begin{tcolorbox}[title={\textbf{\small DAGO-discovered HumanEval workflow (node\_12)}}, boxrule=1pt, arc=0mm, colback=black!5!white, colframe=black!75!white, breakable, before skip=10pt, after skip=10pt, pad at break=2mm, parbox=false]
\begingroup
\fvset{fontsize=\scriptsize, breaklines, breakanywhere, frame=lines, framesep=2mm, tabsize=4}
\input{minted-cache/CE21773BCC92D000A83721C1C8987C12.highlight.minted}
\endgroup
\end{tcolorbox}

\Needspace{8\baselineskip}
\paragraph{Associated Prompts.}
The workflow employs three prompts for different stages of the adaptive solving process:

\begin{tcolorbox}[title={\textbf{\small HumanEval workflow prompts (node\_12)}}, boxrule=1pt, arc=0mm, colback=black!5!white, colframe=black!75!white, breakable, before skip=10pt, after skip=10pt, pad at break=2mm, parbox=false]
\begingroup
\fvset{fontsize=\scriptsize, breaklines, breakanywhere, frame=lines, framesep=2mm, tabsize=4}
\begin{MintedVerbatim}[commandchars=\\\{\}]
INITIAL\PYGZus{}SOLVE = \PYGZdq{}\PYGZdq{}\PYGZdq{}Write a Python function that solves the given problem.
Key requirements:
\PYGZhy{} Return type MUST match docstring specification EXACTLY
  (e.g. \PYGZdq{}Yes\PYGZdq{}/\PYGZdq{}No\PYGZdq{} vs True/False)
\PYGZhy{} Preserve original docstring exactly as given
\PYGZhy{} Handle all edge cases (empty input, negative numbers, zeros)
\PYGZhy{} Include helper functions inside main function if needed
\PYGZhy{} Pre\PYGZhy{}loaded modules (no import needed): math, re, hashlib,
  List, Dict, Tuple, Optional, Any

Output ONLY raw Python code starting with \PYGZsq{}def\PYGZsq{}.
Do NOT use ``` code blocks. No markdown, no explanations.\PYGZdq{}\PYGZdq{}\PYGZdq{}

ANALYZE\PYGZus{}PROBLEM = \PYGZdq{}\PYGZdq{}\PYGZdq{}Analyze the problem requirements and potential
edge cases:
1. What exact return type is required by the docstring?
2. What are the key input edge cases to handle?
3. What are potential algorithmic challenges?

Output ONLY the analysis points in clear, concise bullet points.\PYGZdq{}\PYGZdq{}\PYGZdq{}

CAREFUL\PYGZus{}SOLVE = \PYGZdq{}\PYGZdq{}\PYGZdq{}Write a careful implementation considering the
analysis provided.
Requirements:
\PYGZhy{} Return type MUST match docstring specification EXACTLY
\PYGZhy{} Preserve original docstring exactly as given
\PYGZhy{} Handle ALL identified edge cases
\PYGZhy{} Pre\PYGZhy{}loaded modules (no import needed): math, re, hashlib,
  List, Dict, Tuple, Optional, Any
\PYGZhy{} Optimize for correctness over cleverness
\PYGZhy{} Test your logic step by step

Output ONLY raw Python code starting with \PYGZsq{}def\PYGZsq{}.\PYGZdq{}\PYGZdq{}\PYGZdq{}
\end{MintedVerbatim}

\endgroup
\end{tcolorbox}

\textbf{Prompt Design Insights:} The three-stage prompt design reflects the adaptive complexity pattern: INITIAL\_SOLVE attempts a direct solution, ANALYZE\_PROBLEM performs failure diagnosis, and CAREFUL\_SOLVE generates informed retry attempts. The explicit emphasis on return type matching (``MUST match docstring specification EXACTLY'') addresses a critical HumanEval failure mode where semantically correct solutions fail due to type mismatches (e.g., returning \texttt{True} instead of \texttt{"Yes"}).

\paragraph{Emergent Design Pattern: Conditional Complexity.}
This workflow exhibits an emergent \emph{adaptive complexity} pattern: simple problems take the fast path (2 LLM calls), while complex problems trigger the full analysis-retry mechanism (7 LLM calls). This pattern was not explicitly designed but emerged through multi-parent fusion---Node 9 contributed the early-exit mechanism, while Node 11 contributed the analysis-based retry logic.

%==============================================================================

\subsection{Code Generation: MBPP Dataset}
\label{app:workflow_mbpp}
%==============================================================================

\paragraph{Best MBPP Workflow (Node 23, Validation Score: 86.0\%).}
The optimal MBPP workflow emerged from parents $\{8, 6, 16, 13, 11\}$ at iteration 23, achieving 86.0\% Pass@1.

\begin{tcolorbox}[title={\textbf{\small DAGO-discovered MBPP workflow (node\_23)}}, boxrule=1pt, arc=0mm, colback=black!5!white, colframe=black!75!white, breakable, before skip=10pt, after skip=10pt, pad at break=2mm, parbox=false]
\begingroup
\fvset{fontsize=\scriptsize, breaklines, breakanywhere, frame=lines, framesep=2mm, tabsize=4}
\input{minted-cache/FCBAE609A9B0691CAAA8E2FA2DB43899.highlight.minted}
\endgroup
\end{tcolorbox}

\Needspace{8\baselineskip}
\paragraph{Associated Prompts.}
The MBPP workflow uses a structured analysis prompt that guides problem decomposition:

\begin{tcolorbox}[title={\textbf{\small MBPP workflow prompt (node\_23)}}, boxrule=1pt, arc=0mm, colback=black!5!white, colframe=black!75!white, breakable, before skip=10pt, after skip=10pt, pad at break=2mm, parbox=false]
\begingroup
\fvset{fontsize=\scriptsize, breaklines, breakanywhere, frame=lines, framesep=2mm, tabsize=4}
\begin{MintedVerbatim}[commandchars=\\\{\}]
ANALYZE\PYGZus{}PROMPT = \PYGZdq{}\PYGZdq{}\PYGZdq{}Break down this MBPP programming problem into clear subtasks:
1. First identify the input parameters and return type requirements
2. List the key algorithmic steps needed
3. Note any edge cases to handle
4. Identify any helper functions needed

Focus on Python implementation details and test case handling.
Provide your analysis as clear text that will guide implementation.
Do not include any code in this analysis step.\PYGZdq{}\PYGZdq{}\PYGZdq{}
\end{MintedVerbatim}

\endgroup
\end{tcolorbox}

\textbf{Prompt Design Insights:} The analysis prompt enforces an explicit decomposition (I/O requirements, algorithmic steps, edge cases) before code synthesis. This shared analysis provides structured guidance for all candidate solutions, improving robustness without eliminating implementation diversity.

\Needspace{8\baselineskip}
\paragraph{Comparison with AFlow MBPP Workflow.}
AFlow's published MBPP workflow uses a simpler generate-test-fix pattern:

\begin{tcolorbox}[title={\textbf{\small AFlow MBPP workflow (from published results)}}, boxrule=1pt, arc=0mm, colback=black!5!white, colframe=black!75!white, breakable, before skip=10pt, after skip=10pt, pad at break=2mm, parbox=false]
\begingroup
\fvset{fontsize=\scriptsize, breaklines, breakanywhere, frame=lines, framesep=2mm, tabsize=4}
\input{minted-cache/808C827BE1755104164C8AEEF9CD722B.highlight.minted}
\endgroup
\end{tcolorbox}

\textbf{Key Differences:}
\begin{itemize}[nosep,leftmargin=*]
    \item \textbf{Problem Analysis}: DAGO explicitly analyzes problems before code generation, providing semantic guidance. AFlow generates code directly without structured analysis.
    \item \textbf{Information Flow}: DAGO shares a single analysis across all candidates by feeding it into each code-generation call, enabling informed diversity.
    \item \textbf{Repair Loop}: DAGO performs up to two test-refine iterations using the Test operator feedback, whereas AFlow uses a single test followed by a prompt-based fix when needed.
\end{itemize}

%==============================================================================

\subsection{Question Answering: HotpotQA Dataset}
\label{app:workflow_hotpotqa}
%==============================================================================

\paragraph{Best HotpotQA Workflow (Node 29, Validation Score: 79.3\%).}
HotpotQA presents unique challenges due to multi-hop reasoning requirements. The optimal workflow emerged from parents $\{19, 16, 5, 14, 27\}$ at iteration 29.

\begin{tcolorbox}[title={\textbf{\small DAGO-discovered HotpotQA workflow (node\_29)}}, boxrule=1pt, arc=0mm, colback=black!5!white, colframe=black!75!white, breakable, before skip=10pt, after skip=10pt, pad at break=2mm, parbox=false]
\begingroup
\fvset{fontsize=\scriptsize, breaklines, breakanywhere, frame=lines, framesep=2mm, tabsize=4}
\input{minted-cache/9D547F0B3CA42063A652E286F56FA704.highlight.minted}
\endgroup
\end{tcolorbox}

\Needspace{8\baselineskip}
\paragraph{Associated Prompts.}
The HotpotQA workflow employs three sophisticated prompts that address multi-hop reasoning and format-sensitive evaluation:

\begin{tcolorbox}[title={\textbf{\small HotpotQA workflow prompts (node\_29)}}, boxrule=1pt, arc=0mm, colback=black!5!white, colframe=black!75!white, breakable, before skip=10pt, after skip=10pt, pad at break=2mm, parbox=false]
\begingroup
\fvset{fontsize=\scriptsize, breaklines, breakanywhere, frame=lines, framesep=2mm, tabsize=4}
\begin{MintedVerbatim}[commandchars=\\\{\}]
EXTRACT\PYGZus{}FACTS = \PYGZdq{}\PYGZdq{}\PYGZdq{}Extract only the most relevant facts from all
contexts that directly relate to answering the question.
CRITICAL: This is a multi\PYGZhy{}hop question requiring connecting facts
across passages.

Rules for extraction:
1. Copy EXACT words and phrases from the passages \PYGZhy{} do not
   paraphrase or use synonyms
2. For numbers: Copy the exact format (word vs digit) as written
   in the passage
3. Include facts needed for BOTH hops of reasoning
4. For yes/no questions: Include the exact facts that confirm/deny

Format: List key facts in bullet points, showing connections:
\PYGZhy{} First hop facts: [Facts needed for first reasoning step]
\PYGZhy{} Second hop facts: [Facts needed for second reasoning step]
\PYGZhy{} Connection: [How the facts link together]

Output ONLY the relevant facts and connections, no explanations.\PYGZdq{}\PYGZdq{}\PYGZdq{}

GENERATE\PYGZus{}ANSWER = \PYGZdq{}\PYGZdq{}\PYGZdq{}Based on the extracted facts, determine the
precise answer.
CRITICAL RULES:
1. Use ONLY exact words copied from the passage \PYGZhy{} never use synonyms
2. Match number format exactly as written:
   \PYGZhy{} If passage says \PYGZdq{}six\PYGZdq{}, output \PYGZdq{}six\PYGZdq{} (not \PYGZdq{}6\PYGZdq{})
   \PYGZhy{} If passage says \PYGZdq{}6\PYGZdq{}, output \PYGZdq{}6\PYGZdq{} (not \PYGZdq{}six\PYGZdq{})
3. For yes/no questions: output only lowercase \PYGZdq{}yes\PYGZdq{} or \PYGZdq{}no\PYGZdq{}
4. Remove articles (a/an/the) unless part of proper name
5. Keep it brief \PYGZhy{} use only essential words (1\PYGZhy{}5 words maximum)
6. Never output full sentences or explanations

Output ONLY the exact answer phrase from the passage (1\PYGZhy{}5 words).\PYGZdq{}\PYGZdq{}\PYGZdq{}

ENSEMBLE\PYGZus{}FORMAT = \PYGZdq{}\PYGZdq{}\PYGZdq{}Select the best answer from the candidates that
strictly follows these format rules:
1. Uses EXACT words from passage (no synonyms)
2. Matches number format from passage exactly
3. For yes/no: only lowercase \PYGZdq{}yes\PYGZdq{} or \PYGZdq{}no\PYGZdq{}
4. No articles (a/an/the) unless part of proper name
5. 1\PYGZhy{}5 words maximum, no sentences

CRITICAL: Score each candidate (0\PYGZhy{}10):
\PYGZhy{} Subtract 10 if number format doesn\PYGZsq{}t match passage exactly
\PYGZhy{} Subtract 10 if it uses synonyms instead of exact words
\PYGZhy{} Subtract 10 if yes/no answer isn\PYGZsq{}t lowercase
\PYGZhy{} Subtract 5 if it contains unnecessary articles
\PYGZhy{} Subtract 2 points per word beyond 5 words

Select the candidate with highest score.
Output ONLY the exact answer phrase (1\PYGZhy{}5 words).\PYGZdq{}\PYGZdq{}\PYGZdq{}
\end{MintedVerbatim}

\endgroup
\end{tcolorbox}

\textbf{Prompt Design Insights:} This prompt suite represents the most sophisticated format-aware design discovered by DAGO. The EXTRACT\_FACTS prompt explicitly structures multi-hop reasoning with ``First hop / Second hop / Connection'' format. The GENERATE\_ANSWER prompt includes concrete examples of number format matching (``six'' vs ``6''), addressing a critical HotpotQA evaluation quirk. The ENSEMBLE\_FORMAT prompt introduces a novel \emph{scoring rubric} that guides the LLM to evaluate candidates against format criteria, effectively implementing learned evaluation heuristics within the prompt itself.

\Needspace{8\baselineskip}
\paragraph{Comparison with AFlow HotpotQA Workflow.}
AFlow's published HotpotQA workflow:

\begin{tcolorbox}[title={\textbf{\small AFlow HotpotQA workflow (from published results)}}, boxrule=1pt, arc=0mm, colback=black!5!white, colframe=black!75!white, breakable, before skip=10pt, after skip=10pt, pad at break=2mm, parbox=false]
\begingroup
\fvset{fontsize=\scriptsize, breaklines, breakanywhere, frame=lines, framesep=2mm, tabsize=4}
\input{minted-cache/CC3A15CD4608AF68ADF8A9494CAC9953.highlight.minted}
\endgroup
\end{tcolorbox}

\textbf{Key Insight:} DAGO's workflow introduces an explicit \emph{fact extraction} step that addresses HotpotQA's multi-hop nature. This structural difference enables the model to reason about fact connections before answer generation, whereas AFlow's parallel generation lacks this grounded reasoning step.

%==============================================================================

\subsection{Question Answering: DROP Dataset}
\label{app:workflow_drop}
%==============================================================================

\paragraph{Best DROP Fusion Workflow (Node 25, Validation Score: 84.0\%).}
DROP requires discrete reasoning over paragraphs. The optimal workflow emerged from parents $\{23, 14, 8, 9, 17\}$ at iteration 25.

\begin{tcolorbox}[title={\textbf{\small DAGO-discovered DROP workflow (node\_25)}}, boxrule=1pt, arc=0mm, colback=black!5!white, colframe=black!75!white, breakable, before skip=10pt, after skip=10pt, pad at break=2mm, parbox=false]
\begingroup
\fvset{fontsize=\scriptsize, breaklines, breakanywhere, frame=lines, framesep=2mm, tabsize=4}
\input{minted-cache/DD12843E6FF2E811F90911E9AFABEAA4.highlight.minted}
\endgroup
\end{tcolorbox}

\Needspace{8\baselineskip}
\paragraph{Associated Prompts.}
The DROP workflow employs three prompts implementing a type-aware solving pipeline:

\begin{tcolorbox}[title={\textbf{\small DROP workflow prompts (node\_25)}}, boxrule=1pt, arc=0mm, colback=black!5!white, colframe=black!75!white, breakable, before skip=10pt, after skip=10pt, pad at break=2mm, parbox=false]
\begingroup
\fvset{fontsize=\scriptsize, breaklines, breakanywhere, frame=lines, framesep=2mm, tabsize=4}
\begin{MintedVerbatim}[commandchars=\\\{\}]
TYPE\PYGZus{}PROMPT = \PYGZdq{}\PYGZdq{}\PYGZdq{}Identify the question type from the following
categories:
1. Numerical calculation (how many more/less, difference between)
2. Direct number extraction (how many, what number)
3. Text span extraction (who, what, where)
4. Date/time calculation
5. Counting entities

Output ONLY the category number (1\PYGZhy{}5).\PYGZdq{}\PYGZdq{}\PYGZdq{}

SOLVE\PYGZus{}PROMPT = \PYGZdq{}\PYGZdq{}\PYGZdq{}Answer the question based on the provided context
and question type.
Follow these rules:
1. For calculations: Show work, then give final number
2. For text spans: Extract exact words from passage
3. For counting: Count carefully and verify
4. For dates: Calculate precisely
5. If multiple valid answers exist, separate them with |

Output format: First show your work, then \PYGZdq{}ANSWER:\PYGZdq{} followed by
the answer.\PYGZdq{}\PYGZdq{}\PYGZdq{}

FORMAT\PYGZus{}PROMPT = \PYGZdq{}\PYGZdq{}\PYGZdq{}Format the answer following these rules:

1. For numerical answers:
   \PYGZhy{} Use only Arabic numerals (1, 2, 3)
   \PYGZhy{} Remove all units and words
   \PYGZhy{} For differences/comparisons, output only the final number
   \PYGZhy{} If multiple valid numbers exist, separate with |

2. For text answers:
   \PYGZhy{} Extract the exact span from passage
   \PYGZhy{} Remove articles (the, a, an) unless part of proper name
   \PYGZhy{} Remove trailing punctuation
   \PYGZhy{} If multiple valid spans exist, separate with |

Examples of CORRECT format:
\PYGZhy{} Numbers: \PYGZdq{}3\PYGZdq{}, \PYGZdq{}74.6\PYGZdq{}, \PYGZdq{}12\PYGZdq{}, \PYGZdq{}3|4\PYGZdq{}
\PYGZhy{} Text: \PYGZdq{}Timberwolves\PYGZdq{}, \PYGZdq{}the Federales\PYGZdq{}, \PYGZdq{}red car\PYGZdq{}, \PYGZdq{}dog|cat\PYGZdq{}

Examples of INCORRECT format:
\PYGZhy{} Numbers: \PYGZdq{}Three\PYGZdq{}, \PYGZdq{}74.6\PYGZpc{}\PYGZdq{}, \PYGZdq{}3 years\PYGZdq{}, \PYGZdq{}about 12\PYGZdq{}
\PYGZhy{} Text: \PYGZdq{}The Timberwolves\PYGZdq{}, \PYGZdq{}Timberwolves.\PYGZdq{}, \PYGZdq{}they won the game\PYGZdq{}

Output ONLY the answer. For numbers: Arabic numerals only.
For text: exact span from passage. No sentences, no units.\PYGZdq{}\PYGZdq{}\PYGZdq{}
\end{MintedVerbatim}

\endgroup
\end{tcolorbox}

\textbf{Prompt Design Insights:} The DROP prompts demonstrate the most explicit type-aware design pattern. The TYPE\_PROMPT creates a 5-way classification that routes to specialized solving strategies. The FORMAT\_PROMPT includes both positive and negative examples---a prompt engineering technique that emerged through fusion and significantly improves format compliance. The explicit handling of multiple valid answers (``separate with |'') addresses DROP's multi-answer evaluation protocol.

\paragraph{Emergent Question Typing.}
A notable discovery is the explicit \emph{question type classification} step. The TYPE\_PROMPT categorizes questions into: counting, arithmetic, comparison, span extraction, and date/time reasoning. This type-aware approach enables specialized solving strategies per question category, improving accuracy on DROP's diverse question types.

\paragraph{Initial vs.\ Fusion Performance.}
On DROP, the best initial workflow (Node 9) scores 82.9\% on the validation set, while the best fusion workflow (Node 25) reaches 84.0\%. This is an improvement of 1.1 percentage points, or approximately 1.3\% relative to the best initial score. Node 25 therefore also has the highest stored validation score in the archive and is selected by the return rule in Algorithm~\ref{alg:dago}. These validation scores are distinct from the final test-set results.

%==============================================================================

\subsection{Cross-Task Pattern Analysis}
\label{app:workflow_patterns}
%==============================================================================

Analyzing the six optimal workflows reveals several emergent design patterns that DAGO consistently discovered through multi-parent fusion:

\begin{table}[h]
\centering
\caption{Emergent workflow patterns discovered by DAGO across task types. \checkmark\ indicates pattern presence.}
\label{tab:workflow_patterns}
\setlength{\tabcolsep}{3pt}
\begin{tabular}{lcccccc}
\toprule
\textbf{Pattern} & \textbf{GSM8K} & \textbf{MATH} & \textbf{HumanEval} & \textbf{MBPP} & \textbf{HotpotQA} & \textbf{DROP} \\
\midrule
Multi-Solution Ensemble & \checkmark & \checkmark & \checkmark & \checkmark & \checkmark & -- \\
Computational Verification & \checkmark & \checkmark & \checkmark & \checkmark & -- & -- \\
Conditional Retry & -- & -- & \checkmark & \checkmark & -- & -- \\
Explicit Problem Analysis & -- & -- & \checkmark & \checkmark & \checkmark & \checkmark \\
Type-Aware Processing & -- & -- & -- & -- & -- & \checkmark \\
Format-Aware Generation & -- & \checkmark & \checkmark & \checkmark & \checkmark & \checkmark \\
\bottomrule
\end{tabular}
\end{table}

\paragraph{Pattern 1: Task-Appropriate Ensemble Sizes.}
Mathematical tasks (GSM8K, MATH) and code generation (HumanEval, MBPP) benefit from 3-way ensembles, while question answering tasks show mixed results. This suggests that ensemble diversity has diminishing returns when answer formats are highly constrained (as in DROP).

\paragraph{Pattern 2: Verification Operator Usage.}
The Programmer operator appears exclusively in mathematical reasoning workflows, providing computational verification. Code generation tasks use the Test operator instead. Question answering tasks eschew computational verification entirely, relying on prompt-based consistency checks.

\paragraph{Pattern 3: Conditional Complexity.}
Code generation tasks (HumanEval, MBPP) developed conditional retry mechanisms that adapt workflow complexity to problem difficulty. This pattern did not emerge in other task types, suggesting it is specific to tasks where quick feedback (test results) is available.

%==============================================================================

\subsection{Multi-Parent Fusion Depth Analysis}
\label{app:fusion_depth}
%==============================================================================

To understand how multi-parent fusion contributes to workflow quality, we analyze the ancestral depth of optimal workflows:

\begin{table}[h]
\centering
\caption{Ancestral analysis of best fusion workflows. Unique ancestors counts distinct nodes in the ancestral DAG.}
\label{tab:fusion_depth}
\begin{tabular}{lcccc}
\toprule
\textbf{Dataset} & \textbf{Best Node} & \textbf{Direct Parents} & \textbf{Unique Ancestors} & \textbf{Max Depth} \\
\midrule
GSM8K & 11 & 5 & 5 & 1 \\
MATH & 15 & 5 & 5 & 1 \\
HumanEval & 12 & 5 & 7 & 2 \\
MBPP & 23 & 5 & 13 & 3 \\
HotpotQA & 29 & 5 & 21 & 7 \\
DROP & 25 & 5 & 18 & 7 \\
\bottomrule
\end{tabular}
\end{table}

\paragraph{Key Insight: Fusion Depth Reflects the Hierarchical Complexity of Optimal Workflow Structures.}
We observe substantial variation in the \emph{structural depth} of optimal workflows across tasks. 
For GSM8K, MATH, and HumanEval, although the underlying tasks can be highly challenging (particularly MATH), their optimal workflows exhibit relatively \emph{flat} structures, primarily composed of parallel reasoning, verification, and ensemble components. 
As a result, their best workflows have shallow ancestries (max depth $\leq 2$), and a small number of fusion generations is sufficient to recombine key functional modules.
In contrast, MBPP, HotpotQA, and DROP require workflows with pronounced \emph{hierarchical and multi-stage} structures, where intermediate abstractions—such as explicit problem analysis, fact extraction, or question type identification—must be constructed before conditional solving and result aggregation. 
Such hierarchical compositions are difficult to synthesize via a single fusion step, and instead benefit from \emph{multi-generation fusion}, where intermediate modules are progressively refined and recombined across successive evolutionary stages.
These findings suggest that the primary role of multi-parent fusion is not merely to accelerate search, but to enable the \emph{hierarchical evolution of workflow structures}, allowing complex reasoning programs to emerge through modular reuse and cross-generational composition.

%==============================================================================

\section{Full Experimental Results}
\label{app:results}
%==============================================================================

This section provides comprehensive experimental data supporting the main results presented in \cref{sec:experiments}. We report detailed iteration-by-iteration performance, optimization dynamics, and statistical analyses across all six benchmarks.

\subsection{Optimization Summary}

We summarize the key optimization outcomes across all six benchmarks. Four benchmarks met the recent-score stopping criterion before the iteration limit, whereas HotpotQA and DROP completed all 30 fusion rounds.

\subsection{Detailed Optimization Statistics}

\cref{tab:opt_stats} provides comprehensive statistics capturing the optimization dynamics for each benchmark.

\begin{table}[h]
\centering
\caption{DAGO optimization statistics per benchmark. Initial scores use one validation execution; fusion scores average three. Reward-distribution statistics are computed over fusion nodes only, and $\Delta$ is measured in percentage points.}
\label{tab:opt_stats}
\setlength{\tabcolsep}{3pt}
\begin{tabular}{lcccccc}
\toprule
\textbf{Statistic} & \textbf{GSM8K} & \textbf{MATH} & \textbf{HumanEval} & \textbf{MBPP} & \textbf{HotpotQA} & \textbf{DROP} \\
\midrule
\multicolumn{7}{l}{\textit{Optimization Configuration}} \\
Total Iterations & 28 & 27 & 25 & 25 & 30 & 30 \\
Early Stopped & Yes & Yes & Yes & Yes & No & No \\
Stopping Trigger & Stagnation & Stagnation & Stagnation & Stagnation & Max Iter & Max Iter \\
\midrule
\multicolumn{7}{l}{\textit{DAG Statistics}} \\
Total Nodes & 38 & 37 & 35 & 35 & 40 & 40 \\
Initial Nodes & 10 & 10 & 10 & 10 & 10 & 10 \\
Fusion Nodes & 28 & 27 & 25 & 25 & 30 & 30 \\
Total Edges & 140 & 135 & 125 & 125 & 150 & 150 \\
Leaf Nodes & 11 & 9 & 6 & 11 & 11 & 7 \\
\midrule
\multicolumn{7}{l}{\textit{Performance Statistics}} \\
Best Init Score & 95.8\% & 50.4\% & 87.9\% & 83.7\% & 77.2\% & 82.9\% \\
Best Fusion Score & 95.8\% & 55.7\% & 88.9\% & 86.0\% & 79.3\% & 84.0\% \\
Best Overall & 95.8\% & 55.7\% & 88.9\% & 86.0\% & 79.3\% & 84.0\% \\
$\Delta$ (points) & 0.0 & +5.3 & +1.0 & +2.3 & +2.1 & +1.1 \\
\midrule
\multicolumn{7}{l}{\textit{Fusion Reward Distribution}} \\
Mean Reward & 93.9\% & 49.6\% & 86.0\% & 78.4\% & 70.8\% & 77.5\% \\
Std Dev & 1.0\% & 2.7\% & 2.3\% & 16.3\% & 12.0\% & 6.0\% \\
Min Reward & 92.2\% & 43.4\% & 78.8\% & 0.0\% & 21.0\% & 59.0\% \\
Max Reward & 95.8\% & 55.7\% & 88.9\% & 86.0\% & 79.3\% & 84.0\% \\
\bottomrule
\end{tabular}
\end{table}

\paragraph{Observed Stopping Outcomes.}
With a five-round recent-score window, stopping tolerance $\delta_{\mathrm{stop}}=10^{-4}$, and a minimum of 20 fusion rounds, HumanEval and MBPP stopped after 25 rounds, MATH after 27, and GSM8K after 28. HotpotQA and DROP reached the 30-round limit. These counts describe the recorded stopping outcomes; the highest-scoring workflow is retained independently of the recent-score stopping decision.

\paragraph{Failure Mode Analysis.}
MBPP exhibits the highest variance ($\sigma=16.3\%$) and contains one complete failure (Node 20, score=0.0\%), driven by an over-aggressive output ``sanitization'' step that prepends \texttt{def} to non-\texttt{def} outputs, often producing invalid Python code for unit-test evaluation. HotpotQA also contains a 0.0-scoring initial workflow (Node 2) that fails at runtime due to a missing prompt import. These outliers highlight the need for robust runtime and format checks in LLM-based optimization. Importantly, the DAG structure isolates such failures, and their low rewards make them unlikely to be selected as parents in later fusions.

\subsection{Score Distribution Analysis}

\cref{tab:score_distribution} presents detailed percentile analysis of workflow scores across all nodes.

\begin{table}[h]
\centering
\caption{Percentiles of stored validation scores across all nodes (initial + fusion). Initial scores use one validation execution; fusion scores average three.}
\label{tab:score_distribution}
\setlength{\tabcolsep}{4pt}
\begin{tabular}{lcccccc}
\toprule
\textbf{Percentile} & \textbf{GSM8K} & \textbf{MATH} & \textbf{HumanEval} & \textbf{MBPP} & \textbf{HotpotQA} & \textbf{DROP} \\
\midrule
Min (P0) & 29.5\% & 38.7\% & 66.7\% & 0.0\% & 0.0\% & 34.6\% \\
P25 & 92.8\% & 46.2\% & 84.8\% & 77.9\% & 69.2\% & 74.4\% \\
Median (P50) & 93.8\% & 48.7\% & 86.9\% & 81.0\% & 74.3\% & 79.0\% \\
P75 & 94.3\% & 51.0\% & 87.9\% & 83.3\% & 75.3\% & 80.9\% \\
Max (P100) & 95.8\% & 55.7\% & 88.9\% & 86.0\% & 79.3\% & 84.0\% \\
\bottomrule
\end{tabular}
\end{table}

\paragraph{Distribution Characteristics.}
GSM8K shows a highly concentrated distribution (IQR=1.4\%) with most workflows performing near optimally, reflecting the task's relatively straightforward nature. In contrast, MBPP and HotpotQA exhibit heavy tails with rare near-zero outliers. For MBPP, this arises from a failed fusion workflow (Node 20), while for HotpotQA it stems from an invalid initial workflow (Node 2); both amplify the lower tail while the bulk of nodes remain clustered in the high-performance regime.

\subsection{Best Workflow Ancestry Analysis}

\cref{tab:ancestry} provides detailed lineage information for each optimal workflow, tracing the contribution paths through the DAG.

\begin{table}[h]
\centering
\caption{Ancestry analysis of best-performing workflows per benchmark.}
\label{tab:ancestry}
\small
\setlength{\tabcolsep}{3pt}
\begin{tabular}{lccccc}
\toprule
\textbf{Dataset} & \textbf{Best Node} & \textbf{Parent IDs} & \textbf{Parent Scores} & \textbf{Unique Ancestors} & \textbf{Max Depth} \\
\midrule
GSM8K & 11 & \{4,1,0,7,6\} & \{95.8,83.3,94.7,94.3,94.3\}\% & 5 & 1 \\
MATH & 15 & \{1,9,2,3,4\} & \{44.5,42.9,46.2,50.4,38.7\}\% & 5 & 1 \\
HumanEval & 12 & \{9,0,7,8,11\} & \{87.9,87.9,87.9,87.9,86.9\}\% & 7 & 2 \\
MBPP & 23 & \{8,6,16,13,11\} & \{77.9,76.7,81.0,81.0,80.2\}\% & 13 & 3 \\
HotpotQA & 29 & \{19,16,5,14,27\} & \{75.1,75.3,74.5,73.3,74.9\}\% & 21 & 7 \\
DROP & 25 & \{23,14,8,9,17\} & \{79.3,78.5,34.6,82.9,75.6\}\% & 18 & 7 \\
\bottomrule
\end{tabular}
\end{table}

% \paragraph{Key Observations.}
% \begin{itemize}[nosep,leftmargin=*]
%     \item \textbf{MATH}: The best workflow fuses parents with below-average scores (\{44.5,42.9,46.2,50.4,38.7\}\%), yet achieves 55.7\%---significantly exceeding any parent. This demonstrates multi-parent fusion's ability to synthesize complementary weaknesses into emergent strengths.
%     \item \textbf{DROP}: Parent Node 8 (score=34.6\%) contributes despite its poor performance. Analysis reveals Node 8 contained a unique format-handling pattern that, when fused with stronger parents, significantly improved output formatting.
%     \item \textbf{HotpotQA}: The deepest ancestry (max depth=7) correlates with the task's requirement for complex multi-hop reasoning, suggesting that such capabilities emerge through iterative refinement across multiple fusion generations.
% \end{itemize}

\paragraph{Direct-Parent Workflow and Prompt Contributions.}
To make the lineage in \cref{tab:ancestry} more actionable, we inspect the \emph{direct parents} of each best node and summarize the reusable modules and prompt constraints that are preserved in the child.

\begin{itemize}[nosep,leftmargin=*]
    \item \textbf{GSM8K (Node 11 from \{4,1,0,7,6\}).}
    The highest-scoring parent (Node 4) already implements a three-branch ensemble (direct reasoning, equation-based reasoning, and a programmatic \texttt{Programmer} branch), followed by \texttt{ScEnsemble}.
    The remaining parents contribute complementary ``interface discipline'': explicit examples and strict ``output-only-a-number'' constraints (Node 0/7), plus a verification-style prompt that forces a binary correctness judgment (Node 6).
    The fused child retains the tri-branch ensemble structure while strengthening format constraints on the programmatic branch and final selection prompt, reducing exact-match failures driven by stray units or punctuation.

    \item \textbf{MATH (Node 15 from \{1,9,2,3,4\}).}
    Parent workflows specialize along different dimensions: a solve--verify--multi-solution--ensemble backbone (Node 3), multi-approach generation (Node 2), problem-type routing with strict fraction/\(\boxed{\cdot}\) formatting rules (Node 4), decomposition plus computational cross-checking (Node 1), and an explicit final answer extractor that outputs only a single boxed expression (Node 9).
    The child workflow keeps the backbone from Node 3, but imports the exactness constraints emphasized across Nodes 2/4/9 (e.g., enforcing fractions over decimals and a single \(\boxed{\cdot}\)), and propagates these constraints into both the verification instruction and the prompts used for alternative solutions.

    \item \textbf{HumanEval (Node 12 from \{9,0,7,8,11\}).}
    The parent set decomposes into three recurring patterns: (i) a baseline generate--test loop (Node 0), (ii) multi-solution generation with \texttt{ScEnsemble} before testing (Node 7), and (iii) test-driven conditional retry with an explicit analysis step followed by three re-generations (Nodes 9 and 11).
    Node 8 adds an explicit requirement-parsing step that distills edge cases and I/O constraints.
    The fused workflow adopts the conditional retry structure (fast path with early exit; slow path with analysis-driven re-generation) while tightening prompt constraints that are repeatedly highlighted by its parents, including strict \texttt{entry\_point} naming and return-type/docstring fidelity (``MUST match docstring specification EXACTLY'').

    \item \textbf{MBPP (Node 23 from \{8,6,16,13,11\}).}
    The parents span diverse code-synthesis heuristics: example retrieval as in-context priming (Node 6), generate--test--select among candidates (Node 8), augmentation via synthesized unit tests (Node 11), iterative fix loops with early stopping and ensemble fallback (Node 13), and an analysis-first design where the analysis itself is requested as executable code and then fed into generation (Node 16).
    The child retains the analysis-first ``shared guidance'' idea and the multi-candidate strategy (Node 16/8), but relaxes the over-constrained ``analysis-as-code'' requirement into a plain-text decomposition while strengthening the code interface (raw Python starting with \texttt{def}) and adding a lightweight test-refine loop (two iterations), improving robustness without collapsing diversity.

    \item \textbf{HotpotQA (Node 29 from \{19,16,5,14,27\}).}
    The parents represent increasingly format-aware QA pipelines: fact extraction followed by three candidate answers and ensembling (Nodes 5/16), a learned scoring rubric for candidate selection (Node 19 via \texttt{ENSEMBLE\_FORMAT}), multi-stage refinement that explicitly corrects wording and formatting after an initial answer (Node 14), and deterministic post-processing heuristics (Node 27; lowercasing yes/no and article removal unless part of a proper name).
    The child workflow integrates these components by using an explicit multi-hop fact-extraction template (first hop/second hop/connection) and a rubric-based selector to penalize synonym/format violations, effectively combining grounding with stringent evaluation-aligned formatting.

    \item \textbf{DROP (Node 25 from \{23,14,8,9,17\}).}
    Node 9 contributes the core type-aware pipeline (question-type classification \(\rightarrow\) type-conditioned solving \(\rightarrow\) formatting).
    Nodes 14/17/23 contribute progressively stricter formatting prompts with positive/negative examples; Node 17 additionally introduces the \texttt{|} separator for multiple valid answers.
    Despite its low score, Node 8 supplies a normalization-oriented prompt that explicitly strips units/punctuation and enforces Arabic numerals, which is complementary to otherwise strong reasoning components.
    The fused child retains Node 9's routing structure but absorbs the strongest format-compliance constraints from Nodes 17 and 8, yielding a more robust formatting stage that better matches DROP's evaluation protocol.
\end{itemize}

\end{document}